\documentclass[twoside]{article}

\usepackage[preprint]{aistats2027}
\usepackage{amsmath,amssymb,amsthm,mathtools}
\usepackage{bm}
\usepackage{booktabs}
\usepackage{enumitem}
\usepackage{graphicx}
\usepackage{algorithm}
\usepackage{algpseudocode}
\usepackage{hyperref}
\usepackage[authoryear,round]{natbib}
\usepackage{xcolor}
\usepackage{tikz}
\usetikzlibrary{arrows.meta,positioning,fit,calc}

\hypersetup{
  colorlinks=true,
  linkcolor=blue!55!black,
  citecolor=blue!55!black,
  urlcolor=blue!55!black
} 

\newtheorem{assumption}{Assumption} 
\newtheorem{proposition}{Proposition}
\newtheorem{lemma}[proposition]{Lemma}
\newtheorem{theorem}[proposition]{Theorem}

\newcommand{\E}{\mathbb{E}}
\newcommand{\Prob}{\mathbb{P}}
\newcommand{\ind}[1]{\mathbf{1}\{#1\}}
\newcommand{\R}{\mathbb{R}} 
\newcommand{\calD}{\mathcal{D}}
\newcommand{\calL}{\mathcal{L}}
\newcommand{\bfeta}{\bm{\eta}}
\newcommand{\bftheta}{\bm{\theta}}

\newcommand{\mc}[2]{\ensuremath{#1\,(\mathbin{\pm}#2)}}

\newcommand{\papertitle}{SIRNET: Causal Adaptation of Pretrained Embeddings for Semiparametric Treatment-Effect Estimation Under Surrogacy}


\begin{document}
\runningtitle{SIRNET: Semiparametric Treatment-Effect Estimation}

\twocolumn[
\aistatstitle{\papertitle}
\aistatsauthor{Lei Shi \And Qiuran Lyu}
\aistatsaddress{Adobe Research\\\texttt{leis@adobe.com}
  \And Division of Biostatistics\\UC Berkeley\\\texttt{qiuran\_lyu@berkeley.edu}}
] 

\begin{abstract}
Modern experiments increasingly produce rich post-treatment objects---such as
text, images, and video---that may provide early surrogate signals for outcomes that
are delayed, expensive, or observed only for a subset of participants.  A
pretrained encoder can turn such an object into a strong generic embedding, but
the embedding is not trained for the estimation tasks required by a
causal estimator and may emphasize information irrelevant to the causal target.
We propose SIRNET, a \underline{S}urrogate-\underline{I}nformed \underline{R}epresentation adaptation framework for neural-\underline{NET}work-based causal estimation in the presence of unstructured post-treatment surrogates. Starting from pretrained surrogate
embeddings, SIRNET learns a shared adapter jointly for outcome prediction, treatment propensity, and arm-specific surrogate marginalization, beyond the generic pretraining objective. The resulting covariate and surrogate representations provide a common interface for these nuisance tasks. The final treatment-effect estimate and confidence interval are obtained from cross-fitted orthogonal scores that attain the semiparametric efficiency bound under the stated nuisance-rate conditions.  We also validate SIRNET's performance in synthetic and
semisynthetic experiments, as well as the Anti-Amyloid in Asymptomatic
Alzheimer's Disease (A4) trial, using frozen pretrained brain-MRI embeddings
as surrogates for a long-term outcome. 
\end{abstract}

\section{Introduction}

\paragraph{Background.} Many modern experiments collect rich post-treatment objects for every unit but observe the primary outcome only later or only for a costly subset. These objects can serve as early signals for the primary outcome, and are referred to as surrogates \citep{prentice1989surrogate, athey2026surrogate}. In an online
education experiment \citep{kestin2025ai, pardos2024chatgpt}, for example, treatment may be an AI-assisted feedback
tool, the surrogate an essay written during the course, and the primary outcome
a final standardized test score that could be missing for students who leave
the online course or skip testing. In a clinical trial, treatment may be a new
therapy, early MRI scans may be surrogates, and a later clinical or imaging
endpoint may be unavailable because of dropout, death, or a missed visit \citep{Sperling2014A4}.  Such
text, image, and multimodal objects are far richer than classical scalar
surrogates.  Although pretrained encoders can summarize them effectively for
their source task, the resulting embeddings do not necessarily emphasize the information
needed for causal estimation. In particular, generic embedding directions can be
predictive of visual or linguistic content while being irrelevant, or even
distracting, for the nuisance functions of a causal estimator \citep{veitch2020adapting,ShiBleiVeitch2019Dragonnet}.

This paper studies how to adapt pretrained embeddings of unstructured
surrogates in a single sample with partially observed outcomes: all units in
one sample have baseline covariates $X$, treatment $W$, and a post-treatment
object $S$, but the outcome $Y$ is
partially observed, indicated by a response variable $R$.  The goal is the average treatment effect
\[
  \tau = \E[Y(1)-Y(0)].
\]
The challenge for this problem is three-fold. First, the surrogate $S$ is \emph{post-treatment}:
$S = S(W)$ is the realized potential surrogate under the unit's assigned
arm.  Hence, estimators that condition on the observed $(X_i,S_i)$ without accounting 
for the arm-specific potential-surrogate distribution will be biased. Second, the surrogates are unstructured and high-dimensional. A pretrained embedding
$S$ can be predictive in general while retaining or emphasizing
variation irrelevant to causal effect estimation. Third, given that the outcome is only partially observed, the surrogate can be informative both for predicting the outcome and for modeling its observation.  Without properly accounting for these factors, treatment-effect estimates can be biased and inefficient.

\paragraph{Contributions.} In this paper, we propose SIRNET, a Surrogate-Informed Representation Network for treatment-effect estimation that addresses these challenges. Our central claim is that \emph{multitask causal representation adaptation from pretrained embeddings for unstructured surrogates can improve nuisance estimation for a semiparametric estimator}.  SIRNET learns a shared adaptation of the baseline covariates and post-treatment embeddings for the nuisance functions required by the orthogonal score.  Its outcome, treatment, labeling, and marginal-outcome heads train that adaptation jointly, so the learned coordinates are selected for the causal nuisance tasks rather than the generic pretraining objective.  The final treatment-effect estimate is obtained from cross-fitted orthogonal scores, not a direct neural-network contrast, yielding robustness to model misspecification and statistically valid inference under certain conditions. From a theoretical perspective, we also give sufficient conditions connecting representation-level nuisance approximation to existing semiparametric inference theory. Moreover, we validate SIRNET using extensive synthetic and semisynthetic experiments, as well as a real-world study on the Anti-Amyloid in Asymptomatic Alzheimer's Disease
(A4) study \citep{Sperling2014A4}, which is a Phase~3 randomized trial of solanezumab
in cognitively unimpaired older adults with elevated brain amyloid. The experiments demonstrate improvements in nuisance learning and causal estimation precision compared with baseline methods that do not leverage such surrogate-informed representation adaptation.


\section{Related Work}
\label{sec:related}

\paragraph{Surrogate endpoints and the surrogate index.}
\citet{prentice1989surrogate} formalised the notion of a surrogate endpoint as one that fully mediates the treatment effect on the primary outcome, which is called the \emph{Prentice criterion}. Subsequent work proposed alternative notions and validation criteria, including the proportion of treatment effect explained \citep{freedman1992statistical}, individual- and trial-level measures of surrogacy \citep{buyse1998criteria}, and causal formulations based on potential outcomes and principal stratification \citep{frangakis2002principal}. More recently, 
\citet{athey2026surrogate} propose the \emph{surrogate index}: a
scalar summary $\hat{g}(S)$ trained on a reference population where $Y$
is observed, used as a proxy for $Y$ in a target experiment. 
\cite{KallusMao2024LimitedOutcomeData} study the role of surrogates in the efficient estimation of treatment effects with limited outcome data, which motivates our setup, and derive the efficient influence function for this setting. But the estimator is not directly equipped to adapt generic embeddings through multitask nuisance learning.  SIRNET builds on their work by using the efficient score as its estimation target and adapting pretrained embeddings for all nuisance functions required by the score.

\paragraph{Partially observed outcomes and missing-data methods.} 
Estimating treatment effects when the outcome is missing for some units is a
classical problem.  Complete-case analysis and regression imputation works only under strong assumptions \citep{little2019statistical}.  Inverse-probability-weighted (IPW)
estimators \citep{Rosenbaum1983Propensity} reweight the complete cases by
the reciprocal of the observed-outcome probability $\Pr(R=1\mid W,X)$.
Doubly robust (DR) augmented-IPW estimators
\citep{BangRobins2005DoublyRobust} are consistent if either the outcome
model or the selection model is correct.  \citet{KallusMao2024LimitedOutcomeData} derive the efficient
influence function for treatment effect estimation with limited outcome data, where the
surrogate $S$ can enter both the outcome model and the selection model. They also
show that the efficient score strictly dominates the $X$-only AIPW estimator
whenever $S$ is informative about $Y$ conditional on $X$.  SIRNET uses this
efficient score as its estimation target, with a deep network providing the
nuisance estimates.


\paragraph{Deep learning for treatment-effect estimation.}
Several neural architectures use deep learning to improve treatment-effect estimation. Early work by \citet{johansson2016learning} framed counterfactual prediction as a representation-learning problem. Building on this idea, \citet{ShalitJohanssonSontag2017TARNet} proposed TARNet and CFRNet, which learn a shared representation before splitting into treatment-specific outcome heads. Related approaches use alternative representation-learning objectives, including preservation of local similarity \citep{yao2018representation}, propensity-matched minibatch construction \citep{schwab2018perfect}, and adversarial counterfactual generation \citep{yoon2018ganite}. \citet{ShiBleiVeitch2019Dragonnet} augment a shared outcome representation with a treatment-propensity head and targeted regularisation designed specifically for treatment-effect estimation. More recently, architectures such as FlexTENet allow both shared and treatment-specific components \citep{curth2021inductive}. These methods focus on covariates that are pre-treatment but do not address post-treatment surrogates. A separate line of work learns representations of short-run surrogates for long-run effect estimation. \citet{ChengGuoLiu2020SurrogateRepresentation} learn a sequential surrogate representation designed to accommodate time-varying confounding under a temporal-unconfoundedness assumption. \citet{cai2024long} instead learn latent surrogate representations from observed surrogates and proxies of potentially unobserved surrogates using an identifiable variational autoencoder. These works share SIRNET's motivation of tailoring representations to the causal target rather than using raw surrogate measurements unchanged, but address different identification problems and inferential targets.

\section{Setup and Methodology Foundation}

\subsection{Setup}
\label{sec:setup}

For units $i=1,\ldots,n$, let
\[
  O_i = (X_i,W_i,S_i,R_i,R_iY_i)
\]
denote the observed data.  Here $X_i\in\mathcal{X}$ are pre-treatment
covariates, $W_i\in\{0,1\}$ is a binary treatment, $S_i$ is the
post-treatment object, and
$R_i\in\{0,1\}$ indicates whether the primary outcome is observed. Our methodology development relies on the following assumptions. 

\begin{assumption}[Ignorability and overlap]
\label{ass:ignorability}
Treatment satisfies
\[
  (Y_i(1),Y_i(0),S_i(1),S_i(0)) \perp W_i \mid X_i,
\]
and there exists $\epsilon>0$ such that $\epsilon < e(X_i) < 1-\epsilon$, where
$e(x)=\Prob(W_i=1\mid X_i=x)$.
\end{assumption}

Assumption \ref{ass:ignorability} is standard in causal inference.  It requires that the treatment assignment is independent of the potential outcomes and potential surrogates, given the pre-treatment covariates, which usually holds in randomized experiments.  The overlap condition ensures that there is a positive probability of receiving either treatment for all values of the covariates, which is necessary for estimating the treatment effect across the entire population. These are standard assumptions in causal inference with surrogate outcomes \citep{athey2026surrogate, chen2023semiparametric}, and they are necessary for identifying the average treatment effect.

\begin{assumption}[Conditional missing at random]
\label{ass:missing-at-random}
Outcome labeling is conditionally ignorable:
\[
  R_i \perp Y_i \mid W_i,X_i,S_i,
\]
with $\epsilon < r(W_i,X_i,S_i) < 1-\epsilon$, where
$r(w,x,s)=\Prob(R_i=1\mid W_i=w,X_i=x,S_i=s)$. 
\end{assumption}

The second assumption is the key identifying condition for partial outcome
observation \citep{KallusMao2024LimitedOutcomeData}. This assumption is more debatable in general because it is not guaranteed by randomization.  It requires that the outcome observation mechanism is independent of the outcome itself, conditional on the treatment, covariates, and surrogate. It is an extension of the conditional missing at random (MAR) assumption in missing data literature \citep{little2019statistical}, which allows for the outcome to be missing in a way that depends on observed variables. One argument for its plausibility in the presence of surrogates is that the surrogate $S$ is a rich post-treatment object that may capture information relevant to both the outcome and the outcome observation. Without it, the effect is generally not identifiable and a separate sensitivity analysis framework is needed to address the challenge, which is beyond the scope of the current paper. 

\subsection{Methodological Foundation: Influence Function and Orthogonal Score}

In this section, we briefly review the orthogonal score established by \citet{KallusMao2024LimitedOutcomeData}. Define the nuisance functions
\begin{align*}
  \tilde\mu(w,x,s)
    &= \E[Y_i\mid W_i=w,X_i=x,S_i=s,R_i=1],\\
  \mu(w,x)
    &= \E[\tilde\mu(w,X_i,S_i)\mid W_i=w,X_i=x],\\
  e(x)
    &= \Prob(W_i=1\mid X_i=x),\\
  r(w,x,s)
    &= \Prob(R_i=1\mid W_i=w,X_i=x,S_i=s).
\end{align*}
Let $\bfeta=(\tilde\mu,\mu,e,r)$.  Two familiar identification strategies
motivate the score.  Outcome regression first predicts the observed outcome
from $(W,X,S)$ and then averages that prediction over the surrogate
distribution within each arm:
$\tau=\E[\mu(1,X)-\mu(0,X)]$.  Alternatively, inverse-probability weighting
recovers each arm mean by upweighting units who received that arm and whose
outcome was observed:
\[
  \tau
  =\E\!\left[
    \frac{W R Y}{e(X)r(1,X,S)}
    -\frac{(1-W)R Y}{\{1-e(X)\}r(0,X,S)}
  \right].
\]
The first strategy can be sensitive to outcome-model error, whereas the second
can be unstable when treatment or labeling probabilities are small.  The
efficient influence function combines them.  Specifically, the
efficient score is  
\begin{equation}
  {%
\small
\begin{aligned}
  &\psi(O_i;\tau,\boldsymbol{\eta})\\
  &=
  \mu(1,X_i)-\mu(0,X_i)-\tau
  \\
  &\quad+
  \frac{W_i\left\{
    \frac{R_i
      \bigl(Y_i-\tilde\mu(1,X_i,S_i)\bigr)}{r(1,X_i,S_i)}
    + \tilde\mu(1,X_i,S_i)-\mu(1,X_i)
  \right\}}{e(X_i)}
  \\
  &\quad-
  \frac{(1-W_i)\left\{
    \frac{R_i
      \bigl(Y_i-\tilde\mu(0,X_i,S_i)\bigr)}{r(0,X_i,S_i)}
    + \tilde\mu(0,X_i,S_i)-\mu(0,X_i)
  \right\}}{1-e(X_i)}.
\end{aligned}
}%
\label{eq:orthogonal-score}
\end{equation}
The first line on the right-hand side is the
baseline-adjusted outcome-regression contrast.  Within each arm, the term
$\tilde\mu-\mu$ corrects for how the realized post-treatment surrogate differs
from its arm-specific average, and $R(Y-\tilde\mu)/r$ corrects the surrogate
prediction using the outcomes that are actually observed.  The treatment
weights place both corrections on a common experimental population. The function $\mu(w,x)$ integrates
$\tilde\mu$ over the arm-$w$-specific distribution of the surrogate.  This
marginalization is essential because $S_i=S_i(W_i)$ is post-treatment: the
distribution of $S_i$ observed for treated units $(W_i=1)$ differs from the
distribution that would be observed for the same units under control. 

Under Assumptions \ref{ass:ignorability} and \ref{ass:missing-at-random}, we have
$\E[\psi(O_i;\tau_0,\bfeta_0)]=0$, where $\tau_0$ and $\bfeta_0$ denote the true values.  Moreover, $\psi(O;\tau_0,\bfeta_0)$ is the efficient
influence function for the ATE in the nonparametric observed-data model under Assumptions \ref{ass:ignorability} and \ref{ass:missing-at-random}, as shown by
\citet[Theorem~2.1]{KallusMao2024LimitedOutcomeData}.  Consequently, the
semiparametric efficiency bound is
\begin{equation}
  V^*
  = \operatorname{Var}\!\left\{\psi(O;\tau_0,\bfeta_0)\right\}
  = \E\!\left[\psi(O;\tau_0,\bfeta_0)^2\right].
  \label{eq:efficiency-bound}
\end{equation}
Thus every
regular asymptotically linear estimator has asymptotic variance at least
$V^*$, while an estimator with influence function $\psi$ attains this bound. Moreover, Proposition \ref{prop:score-robustness} states that the score is robust to certain misspecifications of the nuisance functions, providing conditions under which the estimator remains consistent. In particular, when the treatment propensity is known or consistently
estimated, the score is consistent if either the labeling model or the
observed-outcome regression is correct in each arm; the marginal regression
$\mu^\star$ need not itself be correct for unbiasedness in that case.

\section{SIRNET Details}

\subsection{Architecture}
\label{sec:architecture}

Figure~\ref{fig:sirnet-architecture} shows how SIRNET maps the observed data to
the nuisance functions required by Equation~\eqref{eq:orthogonal-score}. It is a structured
multitask learner for the particular collection of nuisances in that score: a
shared adapter lets related prediction tasks inform common embedding
directions, while separate heads and input routes enforce the statistical role
of each nuisance.  The architecture has six blocks.

\textbf{(i) Baseline encoder.}  A neural map converts baseline covariates to
$B_i=b_\theta(X_i)$. All downstream nuisance heads use $B_i$ rather than the raw covariate vector.\textbf{(ii) Surrogate adapter.}  A second map adapts a pretrained post-treatment
  embedding and baseline information to $H_i=g_\theta(B_i,S_i)$.  When $S_i$
  is represented by a vector of AI scores or pretrained embeddings, $g_\theta$ is a
  multilayer adapter that selects and recombines embedding directions for the
  causal nuisance tasks.  
  \textbf{(iii) Treatment head.}  A propensity head estimates
  $\hat e_i=e_\theta(B_i)$.  Under a known
  randomized design this head can be replaced by the known assignment
  probability; estimating it can also provide an auxiliary representation-balancing task. \textbf{(iv) Labeling head.}  A missingness head estimates
  $\hat r_i(w)=r_\theta(w,B_i,H_i)$, the probability of observing the outcome
  given treatment, baseline variables, and the surrogate representation. \textbf{(v) Observed-outcome heads.}  Two treatment-specific heads estimate
  $\hat{\tilde\mu}_i(w)=\tilde\mu_\theta(w,B_i,H_i)$ for $w\in\{0,1\}$.
  These heads are trained only on units with $R_i=1$ but evaluated for all
  units in the orthogonal score. \textbf{(vi) Marginal heads.}  Two heads estimate
  $\hat\mu_i(w)=\mu_\theta(w,B_i)$ for $w\in\{0,1\}$.
  These heads take only baseline covariates as input, omitting the realized
  surrogate $S_i$, thereby approximating the arm-$w$-specific marginal
  expectation of $\tilde\mu_w$ over the distribution of $S_i(w)$ given $X_i$. 


\begin{figure*}[htbp!]
  \centering
  \resizebox{0.8\textwidth}{!}{\begin{tikzpicture}[
  font=\small,
  >=Latex,
  node distance=6mm and 12mm,
  box/.style={
    draw,
    rounded corners=2pt,
    align=center,
    minimum height=9mm,
    text width=28mm,
    fill=white
  },
  widebox/.style={
    draw,
    rounded corners=2pt,
    align=center,
    minimum height=11mm,
    text width=36mm,
    fill=white
  },
  head/.style={
    draw,
    rounded corners=2pt,
    align=center,
    minimum height=9mm,
    text width=36mm,
    fill=white
  },
  stage/.style={
    font=\footnotesize\bfseries,
    align=center
  },
  arrow/.style={->, line width=0.45pt},
  softarrow/.style={->, line width=0.45pt, dashed, gray!75}
]

\node[stage] (stage1) at (0,1.15) {Observed data};
\node[stage, right=18mm of stage1] (stage2) {Representation};
\node[stage, right=28mm of stage2] (stage3) {Nuisance heads};
\node[stage, right=31mm of stage3] (stage4) {Causal layer};

\node[box, fill=blue!7, below=7mm of stage1] (x) {$X_i$\\baseline};
\node[box, fill=green!8, below=7mm of x] (z) { $S_i$\\post-treatment surrogate};
\node[box, fill=gray!8, below=7mm of z] (obs) {$W_i,\ R_i,\ R_iY_i$\\observed labels};

\node[widebox, fill=blue!6, right=of x] (bx) {baseline encoder\\$B_i=b_\theta(X_i)$};
\node[widebox, fill=green!8, right=of z] (hx) {surrogate encoder\\$H_i=g_\theta(B_i,S_i)$};

\node[head, fill=orange!10, right=16mm of bx] (e) {treatment propensity\\$\hat e(B_i)$};
\node[head, fill=orange!10, below=6mm of e] (mu) {marginal index\\$\hat\mu(0,B_i),\ \hat\mu(1,B_i)$};
\node[head, fill=purple!8, below=6mm of mu] (r) {labeling propensity\\$\hat r(0,B_i,H_i)$\\$\hat r(1,B_i,H_i)$};
\node[head, fill=purple!8, below=6mm of r] (yt) {observed-outcome index\\$\hat{\tilde\mu}(0,B_i,H_i)$\\$\hat{\tilde\mu}(1,B_i,H_i)$};

\node[widebox, fill=red!7, right=19mm of mu] (score) {orthogonal ATE layer\\$\hat\tau_{\mathrm{SIRNET}}$};
\node[widebox, fill=gray!8, below=8mm of score] (se) {influence values\\$\hat\psi_i$ and standard error};

\node[widebox, fill=yellow!10, below=12mm of yt, text width=45mm] (loss) {training objective\\outcome + labeling + treatment + marginalization losses};

\draw[arrow] (x) -- (bx);
\draw[arrow] (x.east) -- ++(5mm,0) |- (hx.west);
\draw[arrow] (z) -- (hx);

\draw[arrow] (bx) -- (e);
\draw[arrow] (bx.east) -- ++(6mm,0) |- (mu.west);
\draw[arrow] (hx.east) -- ++(6mm,0) |- (r.west);
\draw[arrow] (hx.east) -- ++(6mm,0) |- (yt.west);

\draw[arrow] (e.east) -- ++(7mm,0) |- (score.west);
\draw[arrow] (mu) -- (score);
\draw[arrow] (r.east) -- ++(7mm,0) |- (score.west);
\draw[arrow] (yt.east) -- ++(7mm,0) |- (score.west);
\draw[arrow] (score) -- (se);

\node[draw, rounded corners=2pt, fit=(bx)(hx), inner sep=3mm,
  label={[font=\footnotesize]above:shared encoders}] {};
\node[draw, rounded corners=2pt, fit=(e)(mu)(r)(yt), inner sep=3mm,
  label={[font=\footnotesize]above:nuisance estimates $\hat\eta$}] (nuisancegroup) {};

\draw[softarrow] (obs.east) -- ++(7mm,0) |- (loss.west);
\draw[softarrow] (loss.north) -- ++(0,4mm) -| (nuisancegroup.south);

\end{tikzpicture}}
  \caption{SIRNET architecture.
  A shared surrogate adapter feeds
  the observed-outcome and labeling-propensity heads, while a baseline encoder
  feeds the treatment and marginal-outcome heads.  The orthogonal ATE
  layer implements equation~\eqref{eq:sirnet-estimator}.}
  \label{fig:sirnet-architecture}
\end{figure*}
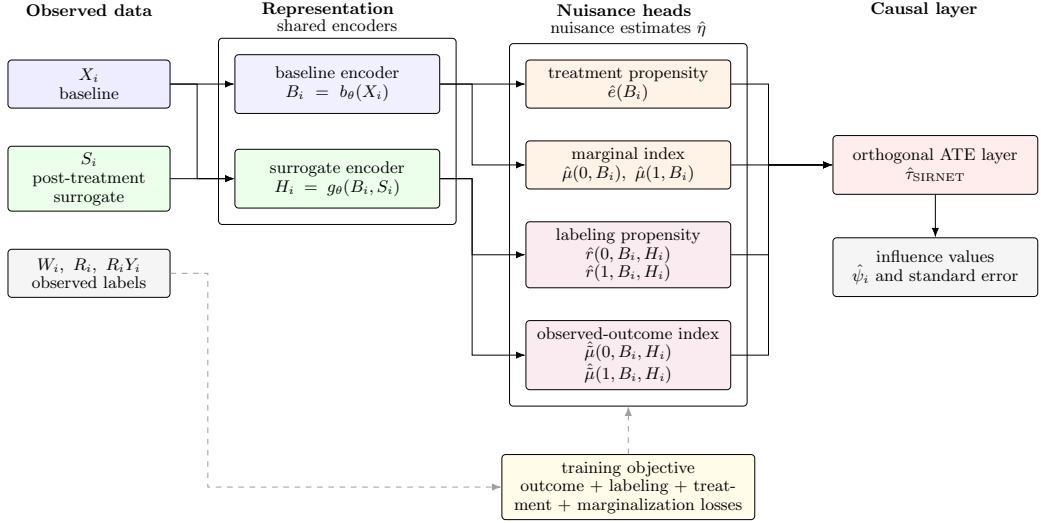

\subsection{Training Objective}

The network is trained by combining predictive losses for each nuisance
function.  Let $\bftheta$ denote all network parameters.  The objective is
\begin{equation}
  \small 
\begin{aligned}
  &\calL(\bftheta)\\
  &=
  \lambda_Y
  \frac{1}{\sum_i R_i}
  \sum_{i=1}^n
    R_i
    \left\{Y_i-\hat{\tilde\mu}_{\bftheta}(W_i,B_i,H_i)\right\}^2
  \nonumber\\
  & +
  \lambda_R
  \frac{1}{n}\sum_{i=1}^n
    \ell_{\mathrm{ce}}\!\left(R_i,\hat r_{\bftheta}(W_i,B_i,H_i)\right)
  +
  \lambda_W
  \frac{1}{n}\sum_{i=1}^n
    \ell_{\mathrm{ce}}\!\left(W_i,\hat e_{\bftheta}(B_i)\right)
  \nonumber\\
  & +
  \lambda_M\sum_{w\in\{0,1\}}
  \frac{
  \sum_{W_i=w}\left\{
    \operatorname{sg}\bigl(\hat{\tilde\mu}_{\bftheta}(w,B_i,H_i)\bigr)-\hat\mu_{\bftheta}(w,B_i)
  \right\}^2}{\sum_i\ind{W_i=w}}
  \label{eq:training-loss}
\end{aligned}
\end{equation}
where $\operatorname{sg}(\cdot)$ denotes stop-gradient.
The first three terms train the outcome, labeling, and treatment nuisance
functions.  The fourth term enforces the further marginalization of the outcome heads by matching $\mu_\theta(w,B_i)$ to a stop-gradiented
copy of $\tilde\mu_\theta(w,B_i,H_i)$, which prevents this loss from updating the outcome-head-specific parameters through its direct pseudo-outcome, but instead build a joint shared-encoder optimization with a detached pseudo-outcome. This choice follows the broader paradigm of learning from detached, model-generated targets. In representation learning, BYOL \citep{grill2020bootstrap} and SimSiam \citep{chen2021exploring} use stop-gradient to prevent direct feedback through prediction targets. Related ideas arise in recursive estimation, including fitted Q-iteration \citep{ernst2005tree} and, more recently, DeepBlip \citep{ma2025deepblip}, which employs detached predictions in recursive neural estimation of causal blip functions. The causal parameter is estimated only after training, by evaluating the cross-fitted orthogonal score.

\subsection{Estimation and Inference}
\paragraph{Orthogonal Estimator and Inference.}

After fitting the neural network, the plug-in orthogonal estimator is
\begin{equation}
  \tiny
\begin{aligned}
  &\hat\tau_{\mathrm{SIRNET}}\nonumber\\
  &=
  \frac{1}{n}\sum_{i=1}^n
  \biggl[
  \hat\mu(1,\hat B_i)-\hat\mu(0,\hat B_i)
  \nonumber\\
  &+
  \frac{W_i
  \left\{
    \frac{R_i}{\hat r(1,\hat B_i,\hat H_i)}
      \bigl(Y_i-\hat{\tilde\mu}(1,\hat B_i,\hat H_i)\bigr)
    + \hat{\tilde\mu}(1,\hat B_i,\hat H_i)-\hat\mu(1,\hat B_i)
  \right\}}{\hat e(\hat B_i)}
  \nonumber\\
  &-
  \frac{(1-W_i)
  \left\{
    \frac{R_i}{\hat r(0,\hat B_i,\hat H_i)}
      \bigl(Y_i-\hat{\tilde\mu}(0,\hat B_i,\hat H_i)\bigr)
    + \hat{\tilde\mu}(0,\hat B_i,\hat H_i)-\hat\mu(0,\hat B_i)
  \right\}}{1-\hat e(\hat B_i)}
  \biggr].
  \label{eq:sirnet-estimator}
\end{aligned}
\end{equation}
For inference, define the estimated influence value
$\hat\psi_i=\psi(O_i;\hat\tau_{\mathrm{SIRNET}},\hat\bfeta_i)$ and
$\bar\psi=n^{-1}\sum_i\hat\psi_i$.  The influence-function standard error is
\begin{equation}
  \widehat{\mathrm{se}}(\hat\tau_{\mathrm{SIRNET}})
  =
  \left\{
    \frac{1}{n(n-1)}\sum_{i=1}^n(\hat\psi_i-\bar\psi)^2
  \right\}^{1/2},
  \label{eq:sirnet-se}
\end{equation}
and an asymptotic $100(1-\alpha)\%$ confidence interval is
$\hat\tau_{\mathrm{SIRNET}}\pm
z_{1-\alpha/2}\widehat{\mathrm{se}}(\hat\tau_{\mathrm{SIRNET}})$.

Moreover, in Appendix~\ref{app:more-estimation}, we introduce additional estimation strategies: regression adjustment and propensity-score weighting, which are complementary to the plug-in orthogonal estimator described above.

\paragraph{Cross-fitting.}

In practice, nuisance networks should be estimated with cross-fitting
\citep{ChernozhukovEtal2018DML}: split the data into $K$ folds, train SIRNET
on all but one fold, predict nuisance functions on the held-out fold, and
compute the estimator only from out-of-fold predictions (Algorithm~\ref{alg:cross-fitted-sirnet}).
Cross-fitting separates nuisance training from score evaluation and allows the
use of flexible machine learning while retaining the orthogonal-score
interpretation. 

\subsection{Theoretical Guarantees}
The oracle result of \citet{KallusMao2024LimitedOutcomeData} gives efficient
inference for any cross-fitted nuisances meeting its product-rate conditions.
Appendix~\ref{app:theory} gives conditional rates for all seven composite
nuisances with a trainable shared encoder, without requiring coordinate
recovery. Its supervised-head bound uses one optimization gap for the joint
supervised objective, rather than separate empirical optimality for each
head. The marginal-head bound additionally requires a final-fit comparison
against fixed generated targets over candidate baseline encoders. Neither
gap follows from gradient-based training or certifies the fixed architectures
used in our experiments.
\section{Numerical Validation}

\subsection{Monte Carlo Validation}
\label{sec:mc}

In this section, we evaluate the proposed methods on synthetic data to verify the statistical
properties.

\paragraph{Data Generating Process}

We use two nonlinear DGPs, both with $n=1000$ to conduct synthetic experiments.
The first uses a construction in which the nuisances have a common nonlinear state along with task-specific components and a protected selection anchor. The second raises the dimensionality and compositional complexity of the nuisance functions and also adds a latent component in the outcome model. The details of the two DGPs can be found in Appendix \ref{app:mc-dgps}.

\paragraph{Results on Estimation and Inference}

We use $5$-fold cross-fitting. The reported methods are complete-case
Difference-in-Means; XGBoost regression imputation with $X$ alone (RI-X) or the realized
$(X,S)$ (RI-XS); the pooled surrogate-prediction contrast (Pooled-Surr);
X-only inverse-probability weighting (IPCW-DiM); the surrogate-only score
with XGBoost or SIRNET nuisances; $X$-only and full-surrogate Double Machine Learning with XGBoost
nuisances; and the full SIRNET score. Appendix~\ref{app:mc-estimators}
defines these estimators. 
Every full-score method estimates the treatment propensity. Imputation and
surrogate-only intervals use paired-bootstrap refits; full-score intervals use
influence-function sandwich standard errors. Table~\ref{tab:mc-main} reports
bias, empirical SD, RMSE, nominal 95\% coverage, and power at the
design-specific evaluation point.

\begin{table*}[ht!]
\centering
\caption{Monte Carlo results. Bias, empirical SD, and RMSE are shown as
estimate ($\pm$ one Monte Carlo standard error). Cov.\ is nominal 95\%
coverage; Pwr.\ is power at the design-specific evaluation point.}
\label{tab:mc-main}
\scriptsize
\setlength{\tabcolsep}{2pt}
\resizebox{\textwidth}{!}{%
\begin{tabular}{@{}lrrrrr@{\quad}rrrrr@{}}
\toprule
& \multicolumn{5}{c}{\textbf{Nonlinear shared+private} ($\tau=0$)}
& \multicolumn{5}{c}{\textbf{Deep bounded-latent} ($\tau=.7$)} \\
\cmidrule(lr){2-6}\cmidrule(lr){7-11}
Method & Bias & SD & RMSE & Cov. & Pwr$_{0.8}$
       & Bias & SD & RMSE & Cov. & Pwr$_{1.0}$ \\
\midrule
\multicolumn{11}{l}{\textit{Complete case}} \\
\quad DiM
  & \mc{-0.100}{0.019} & \mc{0.230}{0.013} & \mc{0.250}{0.014} & $0.927$ & $0.840$
  & \mc{-0.295}{0.022} & \mc{0.275}{0.016} & \mc{0.403}{0.020} & $0.853$ & $0.680$ \\
\midrule
\multicolumn{11}{l}{\textit{Imputation (bootstrap SE)}} \\
\quad RI-X [XGB]
  & \mc{-0.071}{0.018} & \mc{0.224}{0.013} & \mc{0.234}{0.014} & $0.907$ & $0.920$
  & \mc{-0.486}{0.013} & \mc{0.156}{0.009} & \mc{0.511}{0.012} & $0.140$ & $0.793$ \\
\quad RI-XS [XGB]
  & \mc{-0.858}{0.013} & \mc{0.160}{0.009} & \mc{0.872}{0.013} & $0.000$ & $0.160$
  & \mc{-1.049}{0.012} & \mc{0.144}{0.008} & \mc{1.058}{0.012} & $0.000$ & $0.553$ \\
\quad Pooled-Surr [XGB]
  & \mc{+0.786}{0.012} & \mc{0.151}{0.009} & \mc{0.800}{0.012} & $0.000$ & $1.000$
  & \mc{+0.422}{0.015} & \mc{0.187}{0.011} & \mc{0.461}{0.015} & $0.567$ & $1.000$ \\
\midrule
\multicolumn{11}{l}{\textit{IPW only}} \\
\quad IPCW-DiM [XGB]
  & \mc{-0.127}{0.029} & \mc{0.358}{0.021} & \mc{0.379}{0.024} & $0.813$ & $0.873$
  & \mc{-0.299}{0.024} & \mc{0.296}{0.017} & \mc{0.420}{0.021} & $0.780$ & $0.767$ \\
\midrule
\multicolumn{11}{l}{\textit{Partial DR (surrogate path only)}} \\
\quad Surr-Only [XGB]
  & \mc{-0.039}{0.018} & \mc{0.222}{0.013} & \mc{0.225}{0.013} & $0.913$ & $0.953$
  & \mc{+0.495}{0.017} & \mc{0.212}{0.012} & \mc{0.538}{0.017} & $0.280$ & $1.000$ \\
\quad Surr-Only [SIRNET]
  & \mc{-0.008}{0.013} & \mc{0.165}{0.010} & \mc{\textbf{0.164}}{\textbf{0.009}} & $0.953$ & $0.880$
  & \mc{-0.196}{0.017} & \mc{0.206}{0.012} & \mc{0.284}{0.014} & $0.947$ & $0.740$ \\
\midrule
\multicolumn{11}{l}{\textit{Doubly robust}} \\
\quad DML-Xonly [XGB]
  & \mc{-0.006}{0.037} & \mc{0.453}{0.026} & \mc{0.451}{0.048} & $0.960$ & $0.480$
  & \mc{+0.097}{0.029} & \mc{0.357}{0.021} & \mc{0.369}{0.021} & $0.947$ & $0.867$ \\
\quad DML-XS [XGB]
  & \mc{+0.126}{0.029} & \mc{0.509}{0.021} & \mc{0.523}{0.141} & $0.980$ & $0.020$
  & \mc{+0.019}{0.021} & \mc{0.362}{0.015} & \mc{0.362}{0.016} & $0.963$ & $0.480$ \\
\quad SIRNET
  & \mc{-0.069}{0.019} & \mc{0.227}{0.013} & \mc{0.237}{0.013} & $0.933$ & $0.880$
  & \mc{-0.067}{0.022} & \mc{0.266}{0.015} & \mc{\textbf{0.274}}{\textbf{0.016}} & $0.953$ & $0.887$ \\
\bottomrule
\end{tabular}%
}
\end{table*}

Table~\ref{tab:mc-main} shows that SIRNET-based estimators provide the
strongest overall performance across the two DGPs. In the first DGP,
Surr-Only [SIRNET] has the smallest RMSE, reflecting the efficiency gain from
the surrogate path when its outcome bridge is well estimated. In the deeper
selection-linked DGP, the full SIRNET score has the smallest RMSE among the
reported methods, whereas Surr-Only retains big bias. The other methods
either have high bias (for example, RI-XS and Pooled-Surr) or high variance
(for example, IPCW-DiM and DML-Xonly). These results illustrate both the value
of surrogate-informed nuisance learning and the role of the full
outcome-residual correction with propensity weighting when the surrogate-only bridge is imperfect.

\subsection{Semisynthetic Validation}
\label{sec:semisynthetic}

We complement the fully synthetic experiments with two semisynthetic designs that are calibrated to a real clinical trial. The Anti-Amyloid in Asymptomatic Alzheimer's Disease (A4) Study \citep{Sperling2014A4} was a Phase~3 randomized controlled trial that evaluated solanezumab, a humanized monoclonal antibody targeting soluble amyloid-$\beta$, in cognitively unimpaired older adults with elevated cortical amyloid burden on PET imaging.  The trial enrolled 1,169 participants (578 assigned to solanezumab and 591 to placebo). Our analysis sample contains 1,129 participants (552 treated and 577 control): it excludes the 40 participants without an eligible post-treatment composite surrogate at either of the two prespecified surrogate visits.  

We complement the fully synthetic experiments with two semisynthetic designs
that are calibrated to the above real A4 study. In both designs, the $n=1129$ observed
triples $(X_i,W_i,S_i)$ are held fixed and only $(Y_i,R_i)$ are redrawn. The
baseline input $X_i$ contains the standardized tabular variables used in the
application and the first five principal components of the $1,488$-dimensional
baseline Swin embedding \citep{Liu2021Swin}. The post-treatment input $S_i$ contains the first 20
principal components of the 1,489-dimensional composite surrogate embedding.More details on the two semisynthetic designs are provided in Appendix~\ref{app:semisynthetic-dgps}. We use the same
estimator grid and five-fold cross-fitting as in Table~\ref{tab:mc-main}. The
SIRNET and fixed-profile strengthened-panel rows use $n_{\mathrm{sim}}=150$
replications at each signal level. The deep-panel classical and XGBoost rows,
as well as both fold-local DML-XS reruns, use $n_{\mathrm{sim}}=300$. The
results are reported in Table~\ref{tab:semisynthetic}.

\begin{table*}[ht!]
\centering
\caption{Semisynthetic validation. Bias, empirical SD, and RMSE are shown as
estimate ($\pm$ one Monte Carlo standard error). Cov.\ is nominal 95\%
coverage; Pwr.\ is power at the design-specific evaluation point. }
\label{tab:semisynthetic}
\scriptsize
\setlength{\tabcolsep}{2pt}
\resizebox{\textwidth}{!}{%
\begin{tabular}{@{}lrrrrr@{\quad}rrrrr@{}}
\toprule
& \multicolumn{5}{c}{\textbf{Nonlinear shared+private A4} ($\tau=0$)}
& \multicolumn{5}{c}{\textbf{Deep bounded-latent A4} ($\tau=.7$)} \\
\cmidrule(lr){2-6}\cmidrule(lr){7-11}
Method & Bias & SD & RMSE & Cov. & Pwr$_{0.2}$
       & Bias & SD & RMSE & Cov. & Pwr$_{1.0}$ \\
\midrule
\multicolumn{11}{l}{\textit{Complete case}} \\
\quad DiM
  & \mc{-0.024}{0.004} & \mc{0.053}{0.003} & \mc{0.058}{0.003} & $0.953$ & $0.987$
  & \mc{-0.037}{0.003} & \mc{0.048}{0.002} & \mc{0.061}{0.002} & $0.930$ & $1.000$ \\
\midrule
\multicolumn{11}{l}{\textit{Imputation}} \\
\quad RI-X [XGB]
  & \mc{-0.096}{0.004} & \mc{0.052}{0.003} & \mc{0.109}{0.004} & $0.487$ & $0.980$
  & \mc{-0.097}{0.003} & \mc{0.047}{0.002} & \mc{0.107}{0.003} & $0.030$ & $1.000$ \\
\quad RI-XS [XGB]
  & \mc{-0.128}{0.004} & \mc{0.050}{0.003} & \mc{0.137}{0.004} & $0.267$ & $0.933$
  & \mc{-0.126}{0.003} & \mc{0.049}{0.002} & \mc{0.136}{0.003} & $0.003$ & $1.000$ \\
\quad Pooled-Surr [XGB]
  & \mc{-0.253}{0.001} & \mc{0.014}{0.001} & \mc{0.253}{0.001} & $0.000$ & $0.000$
  & \mc{-0.740}{0.001} & \mc{0.016}{0.001} & \mc{0.740}{0.001} & $0.000$ & $0.030$ \\
\midrule
\multicolumn{11}{l}{\textit{IPW only}} \\
\quad IPCW-DiM [XGB]
  & \mc{-0.013}{0.005} & \mc{0.059}{0.003} & \mc{0.060}{0.004} & $0.960$ & $0.987$
  & \mc{-0.014}{0.003} & \mc{0.054}{0.002} & \mc{0.056}{0.002} & $0.983$ & $1.000$ \\
\midrule
\multicolumn{11}{l}{\textit{Partial DR (surrogate path only)}} \\
\quad Surr-Only [XGB]
  & \mc{-0.110}{0.004} & \mc{0.050}{0.003} & \mc{0.121}{0.004} & $0.393$ & $0.900$
  & \mc{-0.107}{0.003} & \mc{0.050}{0.002} & \mc{0.118}{0.003} & $0.083$ & $1.000$ \\
\quad Surr-Only [SIRNET]
  & \mc{-0.025}{0.004} & \mc{0.055}{0.003} & \mc{0.060}{0.003} & $0.933$ & $0.987$
  & \mc{-0.046}{0.004} & \mc{0.051}{0.003} & \mc{0.069}{0.004} & $0.887$ & $1.000$ \\
\midrule
\multicolumn{11}{l}{\textit{Doubly robust}} \\
\quad DML-Xonly [XGB]
  & \mc{+0.006}{0.006} & \mc{0.068}{0.004} & \mc{0.068}{0.004} & $0.973$ & $0.967$
  & \mc{-0.002}{0.003} & \mc{0.060}{0.002} & \mc{0.060}{0.002} & $0.977$ & $1.000$ \\
\quad DML-XS [XGB]
  & \mc{+0.013}{0.004} & \mc{0.066}{0.003} & \mc{0.067}{0.003} & $0.970$ & $0.993$
  & \mc{+0.023}{0.003} & \mc{0.057}{0.002} & \mc{0.061}{0.003} & $0.977$ & $1.000$ \\
\quad SIRNET
  & \mc{-0.019}{0.004} & \mc{0.053}{0.003} & \mc{\textbf{0.056}}{\textbf{0.003}} & $0.967$ & $0.980$
  & \mc{-0.024}{0.004} & \mc{0.051}{0.003} & \mc{\textbf{0.056}}{\textbf{0.003}} & $0.960$ & $1.000$ \\
\bottomrule
\end{tabular}%
}
\end{table*}

From Table~\ref{tab:semisynthetic}, SIRNET has the smallest RMSE in the strengthened semisynthetic DGP ($0.056$ versus $0.067$ for DML-XS [XGB]), with low bias and near-nominal coverage.  SIRNET also has the smallest RMSE in the deep bounded-latent design ($0.056$), although the jointly fit, fold-local XGBoost pipeline is close ($0.061$) and likewise has low bias and near-nominal coverage.  Thus the A4-calibrated experiments show that both flexible nuisance-learning strategies perform well, with SIRNET retaining a modest RMSE advantage in both designs.

\section{Empirical Application: A4 Study}

\subsection{Study Overview}

In this section, we apply the method to the real A4 study. The outcome analyzed here is white matter hyperintensity (WMH) volume, a
measure of cerebrovascular burden visible on T2/FLAIR brain MRI.  WMH is an
established imaging biomarker associated with Alzheimer's-related
neurodegeneration and has been linked to both amyloid pathology and treatment
response in preclinical AD populations.  We define the analysis outcome as
$Y_i=\log(\text{WMH}_i+1)$ at month~66, where $\text{WMH}_i$ is WMH volume
in mm$^3$.  Of the 1129 analysis participants, 797 (70.6\%) had an observed
month~66 WMH measurement; 332 (29.4\%) were missing due to study dropout,
death, or missed visits.  Among these 332, 11 participants (6 solanezumab,
5 placebo) have a recorded death prior to month~66, accounting for a small
share ($3.3\%$) of missingness relative to withdrawal (172) and adverse
events (51); the remaining reasons are listed in
Appendix~\ref{app:a4-diagnostics}.

\subsection{Features and Surrogate Construction}

\textbf{Baseline covariates ($X$).}
Pre-treatment covariates include tabular clinical and demographic variables
(age, sex, years of education, APOE4 allele count, baseline composite amyloid
PET score, and baseline cognitive assessment scores) together with a
high-dimensional imaging representation derived from the screening visit brain
MRI.  The screening MRI is encoded by a frozen pretrained Swin Transformer
\citep{Liu2021Swin}.  Its embedding is compressed to 96 principal components
via PCA. 

\textbf{Post-treatment surrogate ($S$).}
We form a composite surrogate from the latest available eligible
post-treatment MRI visit: week~84 is preferred and week~12 is used as a
fallback.  An eligible visit requires both a Swin embedding and a WMH
measurement.  Of the 1,129 analysis participants, 971 (86.0\%) use the
week~84 visit and 158 (14.0\%) use the week~12 fallback.  The raw surrogate
contains the $1,488$-dimensional frozen embedding and $\log(\mathrm{WMH}+1)$ from the
assigned visit; it is compressed to $96$ principal components within each
training fold, after which SIRNET learns the task-specific surrogate adapter. These post-treatment scans precede the month~66 analysis
endpoint and may capture neurological changes related to later dropout and
outcome evolution, making conditional MAR more plausible, though it is still not testable.

\subsection{Results}

We compare the same set of estimators as in the Monte Carlo and semisynthetic
experiments, detailed in Appendix~\ref{app:mc-estimators}. Table~\ref{tab:a4-results}
shows that the adjusted estimators are broadly consistent with a null effect in
the selected analysis cohort. DiM is more negative ($\hat\tau=-0.092$), whereas
RI-X, the XGBoost DML estimators, and SIRNET are closer to zero and all of their
intervals include zero. Thus the selected-cohort analysis provides no evidence
of a statistically significant solanezumab effect on month-66 WMH volume.  For
the full SIRNET row, configuration selection is nested inside each outer
training fold: the resulting estimate is $-0.044$ (SE $0.035$), compared with
$-0.035$ (SE $0.035$) under the earlier globally fixed configuration.  Thus
the substantive conclusion is insensitive to protecting the outer held-out
scores from tuning.

This finding aligns well with the scientific report for the A4 study. The published main results
\citep{Sperling2023A4Results} found no statistically significant benefit of
solanezumab on PACC or on any secondary cognitive or imaging endpoint.  This
null finding is consistent with the programme of solanezumab trials in
symptomatic patients: both the EXPEDITION 1 and 2 trials in mild-to-moderate
Alzheimer's disease \citep{Doody2014EXPEDITION} and the subsequent
EXPEDITION~3 trial targeting mild dementia \citep{Honig2018EXPEDITION3}
likewise failed to demonstrate significant slowing of disease progression,
suggesting that passive amyloid immunotherapy at the tested doses does not
modify the downstream neurodegenerative cascade.

\begin{table}[htbp!]
\centering
\caption{Estimated ATE of solanezumab on $\log(\text{WMH}+1)$ at month~66.
  }
\label{tab:a4-results}
\begin{tabular}{@{}lrrr@{}}
\toprule
Method & $\hat\tau$ & SE & 95\% CI \\
\midrule
\multicolumn{4}{l}{\textit{Group 1: Complete case}} \\[2pt]
\quad DiM
  & $-0.092$ & $0.055$ & $[-0.200,\;+0.017]$ \\[3pt]
\multicolumn{4}{l}{\textit{Group 2: Imputation (bootstrap SE)}} \\[2pt]
\quad RI-X
  & $-0.011$ & $0.019$ & $[-0.048,\;+0.027]$ \\
\quad RI-XS
  & $-0.007$ & $0.018$ & $[-0.043,\;+0.029]$ \\
\quad Pooled-Surr
  & $-0.061$ & $0.040$ & $[-0.139,\;+0.018]$ \\[3pt]
\multicolumn{4}{l}{\textit{Group 3: IPW only}} \\[2pt]
\quad IPCW-DiM
  & $-0.178$ & $0.147$ & $[-0.466,\;+0.111]$ \\[3pt]
\multicolumn{4}{l}{\textit{Group 4: Partial DR, Surrogate Only}} \\[2pt]
\quad XGB
  & $-0.001$ & $0.021$ & $[-0.042,\;+0.041]$ \\
\quad SIRNET
  & $-0.041$ & $0.027$ & $[-0.094,\;+0.012]$ \\[3pt]
\multicolumn{4}{l}{\textit{Group 5: Doubly robust}} \\[2pt]
\quad DML-Xonly
  & $+0.013$ & $0.040$ & $[-0.065,\;+0.092]$ \\
\quad DML-XS
  & $+0.006$ & $0.028$ & $[-0.049,\;+0.061]$ \\
\quad SIRNET
  & $-0.044$ & $0.035$ & $[-0.112,\;+0.024]$ \\
\bottomrule
\end{tabular}
\end{table}

\section{Discussion}
\label{sec:discussion}

\textbf{Summary.} SIRNET is a neural implementation of a
semiparametric orthogonal moment and a method for causally adapting pretrained
surrogate embeddings.  It applies a Neyman-orthogonal score that uses actual
$Y$ observations to correct the raw surrogate-based estimate, while a shared
adapter pools information across the nuisance tasks that determine that score.
The empirical claim is consequently about improved nuisance estimation from
fixed pretrained embeddings, not about end-to-end representation learning from
raw unstructured objects. In the Monte Carlo ablation, this adaptation reduces
finite-sample variance and RMSE relative to separately trained neural nuisance networks; a
difficulty-matched sharedness continuum (Appendix~\ref{app:negative-transfer})
further shows the shared representation is at least as accurate as
capacity-matched independent networks in both bias and variance across the full
range from independent to fully shared nuisances, with no negative-transfer
penalty.


\textbf{Future directions.} The most immediate next experiment is to train a
pretrained image or text encoder (or a parameter-efficient adapter) inside each
cross-fitting fold and compare frozen generic embeddings, outcome-only
fine-tuning, independent nuisance adapters, and the shared SIRNET adapter.
That experiment would test whether the embedding-level gains documented here
extend to end-to-end adaptation of raw unstructured surrogates.  Other
directions include a sequential neural structure for longitudinal surrogate and
outcome data \citep{shi2024using,ChengGuoLiu2020SurrogateRepresentation,
battocchi2021estimating, shi2025terra}, conditional average treatment effects,
and large-scale multimodal deployments \citep{kaddour2021causal,
shi2025leveraging, hu2026experimentation}.

\textbf{Limitations.}
The conditional MAR assumption requires that all drivers of dropout are captured in $(X_i,S_i)$; unmeasured time-varying confounders would
invalidate identification. The outcome-only approaches require bootstrap for inference, which is time-consuming for deep-learning based methods.  
Most importantly for the ML scope, our A4 inputs are frozen pretrained MRI
embeddings followed by PCA; the study does not show that SIRNET can learn a
useful representation directly from raw images, text, video, or multimodal
objects. We leave these extensions to future work.

\subsection*{AI use statement}

In this work, we used generative AI tools for: 
\begin{quote}
    Design or provide feedback on research methodology or experiments; implement methods; clean and reformat dataset; support qualitative and thematic data analysis; formulate mathematical claims; assist in the writing of proofs; providing LLM judge labels as part of the core research idealization. 
\end{quote}
We have not used generative AI tools for: 
\begin{quote}
    Generate synthetic data sets; help develop theoretical models or conceptual frameworks; provide critical ingredients for proving mathematical claims; propose or refine hypotheses; assist with translation; interpret results.
\end{quote}
and the following is not applicable to this work:
\begin{quote}
    Assist with translation.
\end{quote}
Additionally, we used generative AI tools for:
\begin{quote}
    Create or modify scientific figures or images, suggest experimental parameters, summarize or analyse existing literature, identify relevant literature.
\end{quote}
We have reviewed all AI-assisted work. Specifically: methodology and
experiment designs suggested with AI assistance were assessed by the authors
for soundness and prior art through a manual literature survey before adoption;
AI-assisted code implementing the estimators and experiment harness was read,
run, and tested for correctness by at authors; dataset
cleaning and reformatting scripts were checked against the source data to
confirm no records were dropped or altered; qualitative and thematic analyses
produced with AI assistance were verified against the underlying model outputs
and human labels; and we use AI to help write latex formulas when drafting proofs but guide the process manually to ensure correctness. AI-generated
figures, experimental parameters, and literature pointers were likewise
inspected and confirmed by the authors before inclusion. Also our research target is LLM-as-judges, so we generated and evaluated many LLM judge labels as part of the core research idealization. We take responsibility
for the final content of this work, including text, claims or artifacts produced
with the aid of generative AI.

\bibliographystyle{plainnat}
\bibliography{references}

\section*{Checklist}

\begin{enumerate}

  \item For all models and algorithms presented, check if you include:
  \begin{enumerate}
    \item A clear description of the mathematical setting, assumptions, algorithm, and/or model. \textbf{Yes.} The setup, score, network, and training objective are in Sections~\ref{sec:setup}--\ref{sec:architecture} and Algorithm~\ref{alg:cross-fitted-sirnet}.
    \item An analysis of the properties and complexity (time, space, sample size) of any algorithm. \textbf{Yes.} Appendix~\ref{app:computational-aspects} gives per-fit, cross-fitting, and nested-tuning costs, memory requirements, and effective sample sizes; Appendix~\ref{app:theory} gives statistical analysis on conditional nuisance-rate guarantees.
    \item (Optional) Anonymized source code, with specification of all dependencies, including external libraries. \textbf{Yes.} We attach an anonymized code package as supplementary material, with a README detailing dependencies and instructions for reproducing the experiments.
  \end{enumerate}

  \item For any theoretical claim, check if you include:
  \begin{enumerate}
    \item Statements of the full set of assumptions of all theoretical results. \textbf{Yes.} Identification assumptions are in Section~\ref{sec:setup}; representation, smoothness, and optimization assumptions are in Appendix~\ref{app:assumptions}.
    \item Complete proofs of all theoretical results. \textbf{Yes.} Proofs of the paper's proposition, rate implication, and supporting lemmas appear in Appendix~\ref{app:theory}; externally sourced results are stated with citations.
    \item Clear explanations of any assumptions. \textbf{Yes.} The text following the assumptions discusses conditional MAR, overlap, representation sufficiency, smoothness, and the unverified optimization premise.
  \end{enumerate}

  \item For all figures and tables that present empirical results, check if you include:
  \begin{enumerate}
    \item The code, data, and instructions needed to reproduce the main experimental results (either in the supplemental material or as a URL). \textbf{Yes.} We provide replication files as supplementary material, with a README detailing how to run the experiments.
    \item All the training details (e.g., data splits, hyperparameters, how they were chosen). \textbf{Yes.} Candidate configurations, tuning, and bootstrap settings are in Appendix~\ref{app:hparam}.
    \item A clear definition of the specific measure or statistics and error bars (e.g., with respect to the random seed after running experiments multiple times). \textbf{Yes.} Tables~\ref{tab:mc-main} and~\ref{tab:semisynthetic} define the reported Monte Carlo standard errors, coverage, and power; Appendix~\ref{app:hparam} specifies repetitions and bootstrap intervals.
    \item A description of the computing infrastructure used. (e.g., type of GPUs, internal cluster, or cloud provider). \textbf{Yes.} Appendix~\ref{app:computational-aspects} describes the GPU-equipped server, CPU and memory limits, and execution modes.
  \end{enumerate}

  \item If you are using existing assets (e.g., code, data, models) or curating/releasing new assets, check if you include:
  \begin{enumerate}
    \item Citations of the creator If your work uses existing assets. \textbf{Yes.} The A4 study, Swin architecture, and relevant prior methods are cited in the application and related-work sections.
    \item The license information of the assets, if applicable. \textbf{Yes.} The manuscript specifies the terms governing the A4 data.
    \item New assets either in the supplemental material or as a URL, if applicable. \textbf{Yes.} We attach anonymized code or newly released artifact as supplementary material.
    \item Information about consent from data providers/curators. \textbf{Yes.} The A4 data-access acknowledgement statement is documented in the paper.
    \item Discussion of sensible content if applicable, e.g., personally identifiable information or offensive content. \textbf{No.} We are not aware of any privacy concerns for the participant-level clinical and imaging data.
  \end{enumerate}

  \item If you used crowdsourcing or conducted research with human subjects, check if you include:
  \begin{enumerate}
    \item The full text of instructions given to participants and screenshots. \textbf{Not Applicable.} This paper analyzes existing A4 trial data and reports no new participant-facing task or crowdsourcing study.
    \item Descriptions of potential participant risks, with links to Institutional Review Board (IRB) approvals if applicable. \textbf{No.} The manuscript does not report the relevant A4 ethics approval or secondary-analysis determination and data-protection arrangements.
    \item The estimated hourly wage paid to participants and the total amount spent on participant compensation. \textbf{Not Applicable.} No participants were recruited or paid for this secondary analysis.
  \end{enumerate}

\end{enumerate}

\clearpage
\appendix
\thispagestyle{empty}

\onecolumn
\aistatstitle{Appendix}

\section{More on estimation}\label{app:more-estimation}
 
\subsection{Alternative Estimators}

\textbf{Surrogate-only variant.}
Dropping the $R/r\cdot(Y-\tilde\mu)$ correction from
equation~\eqref{eq:sirnet-estimator} yields a \emph{surrogate-only} score
\begin{align}
  \psi^{\mathrm{surr}}(O_i;\tau,\bfeta)
  &= \mu(1,X_i)-\mu(0,X_i)-\tau
  \nonumber\\ 
  &\quad+ \frac{W_i}{e(X_i)}\bigl[\tilde\mu(1,X_i,S_i)-\mu(1,X_i)\bigr]
  \nonumber\\
  &\quad- \frac{1-W_i}{1-e(X_i)}\bigl[\tilde\mu(0,X_i,S_i)-\mu(0,X_i)\bigr].
  \label{eq:surr-only-score}
\end{align}
This score omits the observed-outcome residual correction and therefore
requires consistent estimation of the outcome regressions
$\tilde\mu$ and $\mu$ (in addition to a correct or known treatment propensity);
it is not protected by a correctly specified labeling propensity when those
regressions are wrong.  Because its influence expression does not account for
first-stage outcome-model estimation, the corresponding plug-in standard error
can understate sampling variability; paired bootstrap confidence intervals are
more reliable for this variant.

\textbf{Weighting-only variant.}
Dropping both outcome-regression terms from
equation~\eqref{eq:sirnet-estimator} yields the inverse-probability-weighted
(IPW) score
\begin{align}
  \psi^{\mathrm{ipw}}(O_i;\tau,e,r) 
  =
  \frac{W_iR_iY_i}{e(X_i)r(1,X_i,S_i)}
  -\frac{(1-W_i)R_iY_i}{\{1-e(X_i)\}r(0,X_i,S_i)}
  -\tau .\label{eq:ipw-only-score}
\end{align}
Its cross-fitted estimator is the sample average of the first two terms in
equation~\eqref{eq:ipw-only-score}, using out-of-fold estimates of $e$ and
$r$. It does not use $\tilde\mu$ or $\mu$, and is consistent when the
treatment and labeling propensities are consistently estimated (with $e$
known under a completely randomized design). Unlike the full score, it has no
outcome-regression augmentation and is not Neyman-orthogonal to errors in the
estimated propensities. It can consequently be highly variable under weak
overlap and sensitive to first-stage propensity error. The usual plug-in
influence-function standard error ignores that first-stage contribution unless
the propensities are known or estimated sufficiently fast; for the fitted-IPW
variant we therefore use paired bootstrap intervals, refitting all nuisance
models within each resample. Probability clipping or weight truncation can
improve numerical stability in practice.

\subsection{Cross-fitting procedure}
Algorithm~\ref{alg:cross-fitted-sirnet} summarizes the cross-fitting procedure for SIRNET.
\begin{algorithm}[ht!]
\caption{Cross-fitted SIRNET estimation}
\label{alg:cross-fitted-sirnet}
\begin{algorithmic}[1]
\Require Study sample $\calD$ and number of folds $K$
\State Partition $\calD$ into folds $\mathcal{I}_1,\ldots,\mathcal{I}_K$.
\For{$k=1,\ldots,K$}
  \State Fit SIRNET on $\calD\setminus\mathcal{I}_k$ using
    loss~\eqref{eq:training-loss}, obtaining training-fold representations.
  \For{each $i\in\mathcal{I}_k$}
    \State Evaluate the fitted model to obtain
      $(\hat{\tilde\mu}_i(0),\hat{\tilde\mu}_i(1),
      \hat\mu_i(0),\hat\mu_i(1),\hat e_i,
      \hat r_i(0),\hat r_i(1))$.
  \EndFor
\EndFor
\State Aggregate the out-of-fold predictions and compute
  $\hat\tau_{\mathrm{SIRNET}}$ via equation~\eqref{eq:sirnet-estimator}.
\State Compute the influence-function standard error in
  equation~\eqref{eq:sirnet-se} and report
  $\hat\tau_{\mathrm{SIRNET}}\pm
  z_{1-\alpha/2}\widehat{\mathrm{se}}(\hat\tau_{\mathrm{SIRNET}})$.
\end{algorithmic}
\end{algorithm}

\section{More details on numerical experiments}

\subsection{Data generating processes}\label{app:mc-dgps}
We used the following two DGPs to generate the Monte Carlo results in Tables~\ref{tab:mc-main}. 

\textbf{DGP 1. Nonlinear shared-plus-private DGP.}
Let $Z_i\sim N(0,I_6)$ be shared factors,
$A_i=(A_i^e,A_i^r,A_i^y)\sim N(0,I_3)$ be task-private factors, and
$C_i\sim N(0,1)$ be a protected selection anchor, all mutually independent.
Fixed dense loading matrices are drawn once and then held fixed across Monte
Carlo replications.  The observed inputs are
\[
 X_i=Z_iB_X+A_iL_X+\varepsilon_{Xi},\qquad
 S_i=Z_iB_S+A_iL_S+W_i\delta+e_1(2C_i+W_i)+\varepsilon_{Si},
\]
where $X_i\in\R^{24}$, $S_i\in\R^{32}$, and the entries of
$\varepsilon_X,\varepsilon_S$ are i.i.d. $N(0,0.5^2)$. The private loading
directions are orthogonal to the shared treatment, labeling, and direct-outcome
directions; the first surrogate coordinate is reserved for the anchor. Let
$q^e_i$ be four variance-standardized dense projections of $X_i$ and $q^r_i$
be four projections of $(X_i,S_i)$ excluding that anchor coordinate. With
$\sigma(t)=(1+e^{-t})^{-1}$, treatment and outcome-observation probabilities are
\begin{align*}
 e_i &= 0.15+0.70\,\sigma\!\left(
   0.75\sin q^e_{i1}+0.55q^e_{i2}q^e_{i3}
   -0.30\tanh q^e_{i4}+0.35\sin A_i^e\right),\\
 r_i &= 0.30+0.60\,\sigma\!\left(
   0.35+0.35\{0.90\sin q^r_{i1}+0.60q^r_{i2}q^r_{i3}
   -0.35\tanh q^r_{i4}\}+0.15\tanh A_i^r-1.20S_{i1}\right).
\end{align*}
The outcome is
\begin{align*}
Y_i={}&1.50\sin(Z_{i1}+0.5Z_{i2})+Z_{i3}Z_{i4}
 +0.70(Z_{i5}^2-1)+0.55\cos Z_{i6}\\
&+0.30X_{i1}-0.20X_{i2}+0.35\{(A_i^y)^2-1\}
 +3\tanh(1.5C_i)+\tau W_i+\varepsilon_{Y,i},
 \qquad \varepsilon_{Y,i}\sim N(0,0.5^2).
\end{align*}
This design tests robustness to task-private structure rather than assuming
complete sharing, while preserving MAR conditional on $(W,X,S)$.

\textbf{DGP 2. Deep bounded-latent shared-plus-private DGP.}
This design uses the same construction with $X\in\R^{48}$ and $S\in\R^{64}$,
independently drawn dense loadings, and anchor proxy
$S_{i1}=0.9C_i+W_i+\varepsilon_{Si1}$. It uses the treatment logit
\[
1.05\sin(q^e_{i1}+0.70\sin q^e_{i2})
+0.80\sin(q^e_{i3}q^e_{i4})+0.45\tanh(q^e_{i1}q^e_{i3})+0.35\sin A_i^e,
\]
within the same $0.15+0.70\sigma(\cdot)$ map, and labeling probability
\[
r_i=0.50+0.50\sigma\!\left(0.30-0.80S_{i1}
+0.85\sin(q^r_{i1}+0.65q^r_{i2})+0.65\sin(q^r_{i3}q^r_{i4})\right).
\]
It adds the deep shared outcome component
\[
1.10\sin\{Z_{i1}+0.75\sin(Z_{i2})\}
+0.85\sin(Z_{i3}Z_{i4})+0.60\tanh(Z_{i5}Z_{i6})
\]
to the outcome model and then a bounded latent interaction:
\[
 V_i=\operatorname{sign}\{\sin(3C_i)+.70\epsilon_i\},\qquad
 Y_i\leftarrow Y_i+1.00V_i+1.35W_iV_i.
\]
Neither $C_i$ nor $V_i$ is supplied directly to an estimator. The bounded,
mean-zero interaction leaves the population ATE equal to $\tau$ but makes an
outcome-only bridge sensitive to representation error. We report estimation
metrics at $\tau=.7$ and power at $\tau=1$.

\subsection{Estimators compared in the Monte Carlo experiments}
\label{app:mc-estimators}

This subsection defines the method names in Table~\ref{tab:mc-main}. All
machine-learned quantities below are cross-fitted; hats therefore denote
out-of-fold predictions. Write $\mathbb P_n a=n^{-1}\sum_i a_i$ and
$n_w=\sum_i\ind{W_i=w}$. For a fitted conditional-outcome
model, write $\hat m_w(F)=\widehat{\E}[Y\mid W=w,R=1,F]$.

\paragraph{Complete case and imputation baselines.}
\textbf{DiM} is the unadjusted complete-case contrast,
\[
 \hat\tau_{\mathrm{DiM}}=
 \frac{\sum_i W_iR_iY_i}{\sum_iW_iR_i}
 -\frac{\sum_i(1-W_i)R_iY_i}{\sum_i(1-W_i)R_i}.
\]
\textbf{RI-X} is $\mathbb P_n\{\hat m_1(X_i)-\hat m_0(X_i)\}$, where each
outcome model is fitted within treatment arm on observed outcomes. 
\textbf{RI-XS} uses the same formula with $F_i=(X_i,S_i)$. It is deliberately
included as a naive benchmark: it evaluates both arm-specific outcome models
at the realized post-treatment surrogate $S_i(W_i)$, rather than marginalizing
over the surrogate distribution under each arm.

\textbf{Pooled-Surr} fits one pooled model
$\hat\nu(X,S)\approx\E[Y\mid X,S,R=1]$, ignoring treatment, and contrasts
its predictions between the realized treatment groups:
\[
 \hat\tau_{\mathrm{Pooled\text{-}Surr}}
 =\frac{1}{n_1}\sum_iW_i\hat\nu(X_i,S_i)
  -\frac{1}{n_0}\sum_i(1-W_i)\hat\nu(X_i,S_i).
\]

\paragraph{Missingness-only and surrogate-only scores.}
\textbf{IPCW-DiM} estimates arm-specific X-only labeling models
$\hat r_w(X)\approx\Pr(R=1\mid W=w,X)$ and uses
\[
 \hat\tau_{\mathrm{IPCW}}=
 \frac{1}{n_1}\sum_i\frac{W_iR_iY_i}{\hat r_1(X_i)}
 -\frac{1}{n_0}\sum_i\frac{(1-W_i)R_iY_i}{\hat r_0(X_i)}.
\]
It cannot correct selection that remains associated with $S$ after
conditioning on $X$.

For the next two methods, let $\hat e(X)$ be the treatment propensity,
$\hat{\tilde\mu}_w(X,S)$ the observed-outcome regression, and
$\hat\mu_w(X)$ its arm-specific marginalization over $S$. 
\textbf{Surr-Only} uses
\[
 \mathbb P_n\!\left[\hat\mu_1(X_i)-\hat\mu_0(X_i)
 +\frac{W_i}{\hat e(X_i)}\{\hat{\tilde\mu}_1(X_i,S_i)-\hat\mu_1(X_i)\}
 -\frac{1-W_i}{1-\hat e(X_i)}\{\hat{\tilde\mu}_0(X_i,S_i)-\hat\mu_0(X_i)\}\right].
\]
It omits the $R/r$ outcome-residual correction. ``Surr-Only [XGB]'' and
``Surr-Only [SIRNET]'' differ only in the nuisance learner.

\paragraph{Full orthogonal scores.}
\textbf{DML-Xonly [XGB]} uses the ordinary missing-outcome AIPW score with
all nuisances conditioned on $X$ only:
\[
 \mathbb P_n\!\left[\hat m^X_1(X_i)-\hat m^X_0(X_i)
 +\frac{W_iR_i}{\hat e(X_i)\hat r^X_1(X_i)}\{Y_i-\hat m^X_1(X_i)\}
 -\frac{(1-W_i)R_i}{\{1-\hat e(X_i)\}\hat r^X_0(X_i)}
 \{Y_i-\hat m^X_0(X_i)\}\right].
\]
\textbf{DML-XS [XGB]} and
\textbf{SIRNET} use the full limited-outcome-data orthogonal score in
equation~\eqref{eq:sirnet-estimator}; the former estimates its nuisance
functions with XGBoost, while the latter estimates them jointly with SIRNET.
Thus DML-XS and SIRNET have the same target score and differ in nuisance
estimation, whereas DML-Xonly omits the surrogate from its outcome and
labeling nuisances. All estimated probabilities are clipped using the
prespecified treatment- and labeling-probability floors before score
evaluation.

\subsection{Architecture-ablation variants}
\label{app:additional-architectures}

In this section, we evaluate different architecture-ablation variants of SIRNET to understand the impact of shared and task-specific components on performance.
Table~\ref{tab:mc-head-robustness} holds the observed data, cross-fitting
folds, optimizer, loss weights, training budget, early-stopping rule, and
probability-clipping rule fixed, and evaluates every variant with the same full
orthogonal score. The variants differ only in how nuisance tasks share neural
parameters and in the explicitly indicated raw-feature residual paths.

\paragraph{Shared.}
This is the default SIRNET. A baseline encoder produces $B=b_\theta(X)$ and
a surrogate encoder produces $H=g_\theta(B,S)$. The treatment and marginal
heads use $B$; the labeling and arm-specific outcome heads use $(B,H)$. Thus
all nuisance tasks shape the same two encoders.

\paragraph{Partial.}
Partial sharing retains the common $B$ and $H$ encoders, but gives the
outcome and labeling tasks separate residual adapters,
$H_Y=H+a_Y(H)$ and $H_R=H+a_R(H)$, before their respective heads. Treatment
and marginal heads continue to use $B$. This tests whether the outcome and
missingness tasks benefit from a small task-private adjustment while preserving
the shared representation.

\paragraph{Wide/deep.}
Wide/deep retains the shared deep encoders and heads, and adds zero-initialized
linear residual paths from the standardized raw inputs: $X$ for treatment and
marginal heads, and $(W,X,S)$ or $(X,S)$ as appropriate for labeling and
outcome heads. It therefore tests whether a shallow direct-feature path adds
useful task-private capacity beyond the shared nonlinear representation. The
residual paths are additive, not replacements for the shared heads.

\paragraph{Partial+wide/deep.}
This variant combines the outcome/labeling residual adapters of Partial with
the raw-feature linear residual paths of Wide/deep.

\paragraph{Separate.}
The Separate baseline has no shared trainable parameters. It fits four
independent MLP blocks for $e(X)$, $r(W,X,S)$,
$\tilde\mu(0,X,S),\tilde\mu(1,X,S)$, and
$\mu(0,X),\mu(1,X)$, respectively, and then evaluates the same full score.
Its marginal block is trained against detached outcome predictions, so it
retains the generated-target construction without an encoder-mediated gradient
path. It is matched as closely as practical in trainable-parameter count and
training budget: 5,426 parameters versus 5,255 for the strengthened-panel
shared model, and 20,446 parameters in the deep panel, matched to the wider
main-table SIRNET configuration. Consequently, Separate is an ablation of
representation sharing, not a different causal estimator.

\begin{table}[htbp!]
\centering
\caption{Central test of shared nuisance learning ($n=1000$, $K=5$,
  $n_{\mathrm{sim}}=150$ per panel). All rows use the same full orthogonal
  score, input restrictions, and generated marginalization target.
  \textbf{Separate} trains four independent neural nuisance models with no
  shared parameters, compute/parameter-matched to full-score SIRNET.}
\label{tab:mc-head-robustness}
\small
\setlength{\tabcolsep}{5pt}
\begin{tabular}{@{}lrrr@{}}
\toprule
Architecture & Bias & SD & RMSE \\
\midrule
\multicolumn{4}{c}{\textbf{Nonlinear shared+private DGP} ($\tau=0$)} \\
\midrule
Shared             & $-0.069$ & $0.227$ & $0.237$ \\
Partial            & $-0.080$ & $0.226$ & $0.239$ \\
Wide/deep          & $-0.060$ & $0.234$ & $0.240$ \\
Partial+wide/deep  & $-0.076$ & $0.238$ & $0.249$ \\
Separate           & $-0.032$ & $0.270$ & $0.272$ \\
\midrule
\multicolumn{4}{c}{\textbf{Deep bounded-latent shared+private DGP} ($\tau=.7$)} \\
\midrule
Shared             & $-0.079$ & $0.251$ & $0.262$ \\
Partial            & $-0.097$ & $0.267$ & $0.283$ \\
Wide/deep          & $-0.022$ & $0.261$ & $0.261$ \\
Partial+wide/deep  & $-0.026$ & $0.268$ & $0.268$ \\
Separate           & $+0.001$ & $0.331$ & $0.331$ \\
\bottomrule
\end{tabular}
\end{table}

From Table~\ref{tab:mc-head-robustness}, the shared architecture has the smallest RMSE in the top panel. In
the bottom panel, wide/deep is only marginally smaller than shared ($0.261$
versus $0.262$), whereas partial sharing is less precise. This is the anticipated variance--RMSE mechanism: sharing
does not make every individual nuisance estimate uniformly less biased, but it
can stabilize the collection of nuisances entering the orthogonal score.
Overall, the comparison provides direct evidence for the statistical multitask
claim, rather than merely showing that a larger neural network works well.

\subsection{Ablation on sharedness strength}
\label{app:negative-transfer}

The main Monte Carlo designs give the nuisance functions both shared and
task-specific structure. Here we ask whether sharing network parameters still
helps when the response and outcome tasks have less in common. We compare
shared SIRNET, a partially shared network, and separate networks while varying
the amount of shared signal.

\paragraph{Difficulty-matched design.}
Simply weakening a shared signal would also weaken the selection problem. To
separate these changes, we use a design indexed by $\gamma\in[0,1]$. The
response and outcome tasks each use an index of the form
$\sqrt{\gamma}\,f(Z_{\mathrm{core}})+\sqrt{1-\gamma}\,f(Z_k)$, where the
shared and task-specific terms are independent and have unit variance.
Thus each index has the same variance for every $\gamma$: at $\gamma=1$
the two tasks use the same latent index, while at $\gamma=0$ they use
independent indices. The treatment task uses a separate index throughout.
A fixed latent factor $U$ affects the outcome and is visible to the response
model only through a noisy, treatment-shifted surrogate coordinate. This
maintains selection bias while preserving MAR conditional on $(W,X,S)$ and
the treatment effect $\tau$. Across the grid, complete-case bias stays
between $-0.13$ and $-0.20$, and the outcome SD is about $3.15$. These
diagnostics suggest that the estimation problem remains broadly comparable
as sharing varies.

\paragraph{Response calibration.}
We keep the true response probability in $[0.35,0.90]$. We also clip estimated response probabilities at $0.25$,
below the true lower bound, to limit the effect of very small predictions.

\paragraph{Results.}
Figure~\ref{fig:mc-negative-transfer} reports bias, SD, and RMSE of the full
orthogonal score against $\gamma$ for the three architectures ($n=8000$,
$n_{\mathrm{sim}}=120$, $K=5$ folds; shared $4871$, partial $6023$, and
separate $4773$ trainable parameters). Shared and partial networks have
lower bias and RMSE than separate networks at every tested $\gamma$,
including $\gamma=0$. Shared-model bias ranges from $-0.010$ to $-0.020$
versus $-0.029$ to $-0.038$ for separate networks; the corresponding RMSE
ranges are $0.073$--$0.084$ and $0.076$--$0.090$. Coverage ranges from
$0.93$ to $0.98$. As a diagnostic, we replace estimated treatment and
response probabilities with their true values while retaining the learned
outcome heads. The resulting bias is small ($+0.004$ to $-0.010$) and
similar across architectures. A slight increase in bias near $\gamma=1$
appears for all methods, including this diagnostic; these results do not
isolate its cause.

\begin{figure}[htbp!]
\centering
\includegraphics[width=\textwidth]{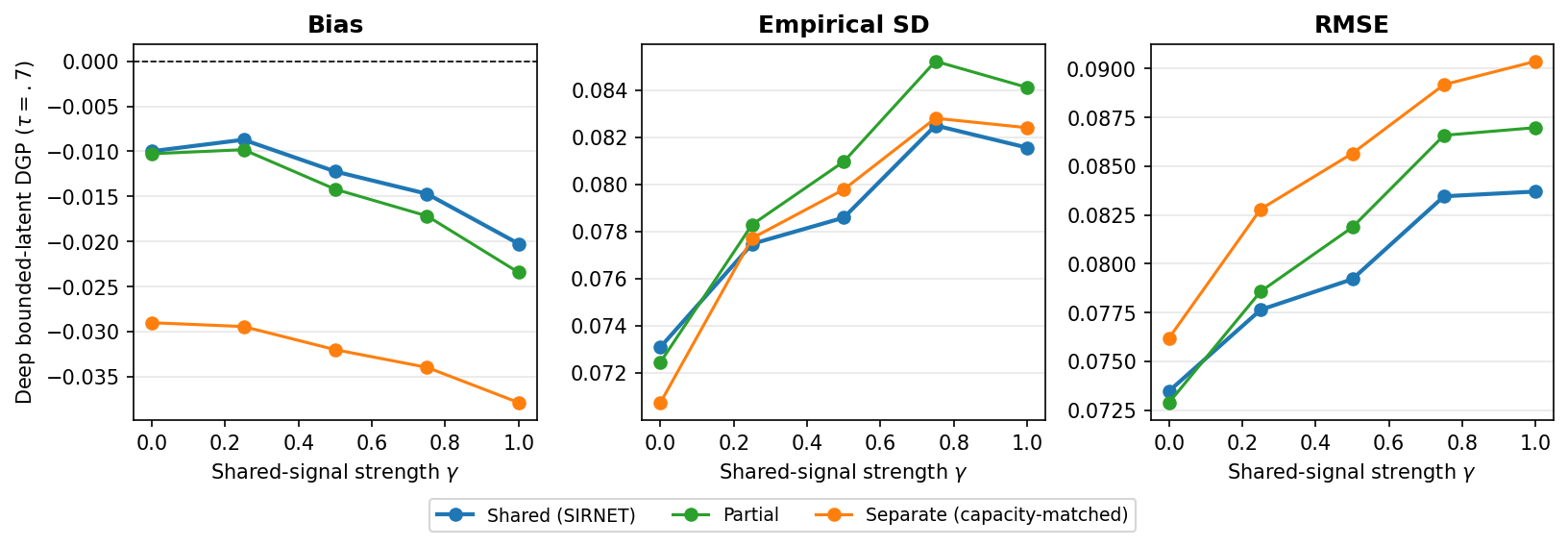}
\caption{Difficulty-matched sharedness continuum ($\tau=.7$, $n=8000$,
  $n_{\mathrm{sim}}=120$, $5$ folds). Bias, empirical SD, and RMSE of the full
  orthogonal score for the shared (SIRNET), partial, and capacity-matched
  separate architectures as a function of the sharedness parameter $\gamma$,
  where $\gamma=1$ makes the response and outcome nuisances a common latent
  index and $\gamma=0$ makes them independent. }
\label{fig:mc-negative-transfer}
\end{figure}

\paragraph{Interpretation.}
In this design, sharing does not incur a large loss even when
the response and outcome signals are independent, although it is slighly worse than but almost comparable to separate networks in variance when $\gamma=0$. Together
with Table~\ref{tab:mc-head-robustness}, the result supports shared nuisance
learning for the settings tested here, especially for its ability to maintain a good balance between bias and variance control.

\subsection{High-Dimensional Representation-Reuse Diagnostics}
\label{app:repr-head-diagnostic}

In this section we ask a diagnostic question: can a
trained shared encoder be handed to a different nuisance learner?
We compare raw, shared-SIRNET, and
true (oracle) representations in the same five-fold
cross-fitted orthogonal score, and compare them with two nuisance learners: Ridge and XGBoost. 
The results are reported in Table \ref{tab:mc-repr-quality}. We can see that overall the transfer of 
the learned representation is beneficial for Ridge but not for XGBoost: XGB/SIRNET has RMSE $0.384$ versus
$0.336$ in the first panel and $0.390$ versus $0.332$ in the second
panel. A paired 30-replication screen with wider $(32,32)$ and $(32,64)$
embeddings and a deeper, slower XGBoost head gave similar results. The
diagnostic therefore supports a narrow conclusion: SIRNET's encoder can be
useful to some replacement heads, but it is not a universally transferable
feature map for tree learners.

\begin{table}[htbp!]
\centering
\caption{High-dimensional representation-reuse diagnostic ($n=1000$, $K=5$,
  $p_{\mathrm{noise}}=40$, $n_{\mathrm{sim}}=300$ per panel). All rows use
  the double-robust score; only point-estimation metrics are reported.}
\label{tab:mc-repr-quality}
\small
\setlength{\tabcolsep}{5pt}
\begin{tabular}{@{}lrrr@{}}
\toprule
Representation / learner & Bias & SD & RMSE \\
\midrule
\multicolumn{4}{c}{\textbf{Nonlinear shared+private DGP} ($\tau=0$)} \\
\midrule
Ridge / raw    & $-0.046$ & $1.216$ & $1.215$ \\
XGB / raw      & $+0.114$ & $0.317$ & $0.336$ \\
Ridge / SIRNET & $-0.057$ & $0.268$ & $0.273$ \\
XGB / SIRNET   & $-0.057$ & $0.380$ & $0.384$ \\
Ridge / oracle & $-0.012$ & $0.356$ & $0.356$ \\
XGB / oracle   & $+0.118$ & $0.451$ & $0.466$ \\
\midrule
\multicolumn{4}{c}{\textbf{Deep bounded-latent shared+private DGP} ($\tau=.7$)} \\
\midrule
Ridge / raw    & $-0.068$ & $2.654$ & $2.651$ \\
XGB / raw      & $-0.001$ & $0.332$ & $0.332$ \\
Ridge / SIRNET & $-0.102$ & $0.290$ & $0.308$ \\
XGB / SIRNET   & $-0.110$ & $0.375$ & $0.390$ \\
Ridge / oracle & $-0.053$ & $0.832$ & $0.832$ \\
XGB / oracle   & $+0.020$ & $0.341$ & $0.341$ \\
\bottomrule
\end{tabular}
\end{table}

\subsection{Sample-Size Sensitivity}
\label{sec:semisyn-nsizes}

We revisit the two Monte Carlo DGPs with a varying sample size experiment to assess the performance of several strong estimators from Table \ref{tab:mc-main}.  At each
$n\in\{800,1000,1200,1400,1600,1800\}$, we run $n_{\mathrm{sim}}=300$
independent draws with the DGP functions, network capacity, cross-fitting,
and clipping rules held fixed.  We
retain the two full doubly robust scores (DML-Xonly and DML-XS) together with
full SIRNET and Surr-Only [SIRNET], and report bias, empirical SD, and RMSE. 

\begin{figure*}[ht!]
  \centering
  \IfFileExists{figures/mc-nsizes-strengthened.png}{%
    \includegraphics[width=\textwidth]{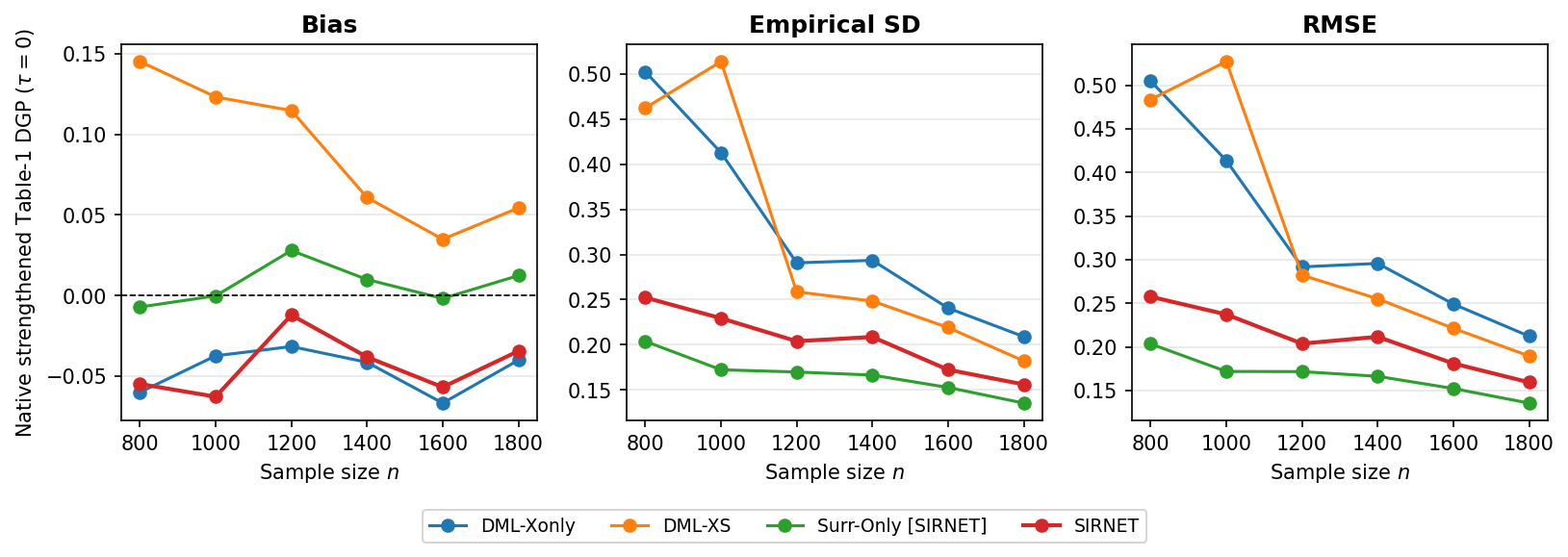}%
  }{\fbox{\parbox[c][2.0in][c]{0.94\textwidth}{\centering
      Strengthened native-DGP sample-size figure unavailable.}}}
  \caption{Sample-size sensitivity for the first DGP.  Four full-score or SIRNET-based estimators are evaluated at
  six sample sizes with $n_{\mathrm{sim}}=300$ per cell.}
  \label{fig:mc-nsizes-strengthened}
\end{figure*}

In the first DGP, the two SIRNET scores improve steadily with sample
size.  Full SIRNET's RMSE declines from $0.258$ at $n=800$ to $0.160$ at
$n=1800$, while Surr-Only [SIRNET] declines from $0.204$ to $0.136$ with
near-zero bias.  Both XGBoost full scores also benefit from larger samples,
but remain less precise at $n=1800$ (RMSE $0.213$ for DML-Xonly and $0.190$
for DML-XS).

\begin{figure*}[ht!]
  \centering
  \IfFileExists{figures/mc-nsizes-deep.png}{%
    \includegraphics[width=\textwidth]{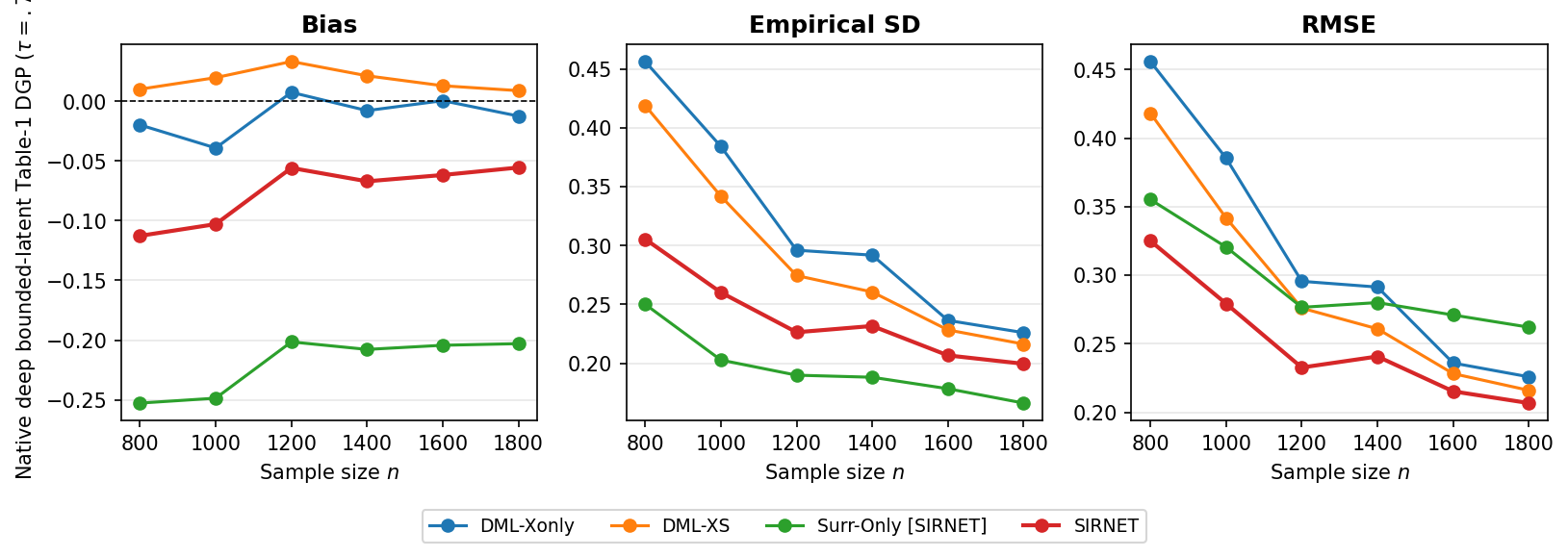}%
  }{\fbox{\parbox[c][2.0in][c]{0.94\textwidth}{\centering
      Deep native-DGP sample-size figure unavailable.}}}
  \caption{Sample-size sensitivity for the deep bounded-latent DGP, using the same four estimators,
  six sample sizes, and $n_{\mathrm{sim}}=300$ per cell.}
  \label{fig:mc-nsizes-deep}
\end{figure*}

The deep DGP makes the response correction consequential.  Full SIRNET's RMSE
falls from $0.325$ to $0.207$, below DML-XS ($0.418$ to $0.216$) and DML-Xonly ($0.456$ to $0.226$).  Surr-Only [SIRNET] is less variable but retains substantial negative bias, from $-0.253$ to $-0.203$, whereas full SIRNET reduces it from $-0.113$ to $-0.055$.  Thus larger samples improve nuisance learning, but do not substitute for the full response/outcome correction in the selection-linked mechanism.

\subsection{Semisynthetic DGPs}
\label{app:semisynthetic-dgps}

\paragraph{Nonlinear shared-plus-private DGP.}
To adapt the architecture test of Section~\ref{sec:mc} to the
A4 inputs, let $V_i$ concatenate the standardized $X_i$ and $S_i$ components.
Let $b^{\mathrm{MRI}}_{i1}$ and $b^{\mathrm{MRI}}_{i2}$ be the first two
standardized principal components of the baseline Swin MRI embedding; these
are included in $X_i$.
Six fixed dense orthogonal projection directions define standardized shared
factors $Z_{i1},\ldots,Z_{i6}$. Two additional projection directions,
orthogonal to the shared directions, define outcome-private and
labeling-private factors $A_i^y$ and $A_i^r$. A protected selection direction
is the cross-fitted residual of the first surrogate PC after predicting it
from $X_i$. The base outcome is
\begin{align*}
 f_i^{\mathrm{SP}}={}&1.50\sin(Z_{i1}+0.5Z_{i2})+Z_{i3}Z_{i4}
 +0.70(Z_{i5}^2-1)+0.55\cos Z_{i6}\\
 &+0.30b^{\mathrm{MRI}}_{i1}-0.20b^{\mathrm{MRI}}_{i2}
 +0.35\{(A_i^y)^2-1\}
 +0.70\tanh(1.5s_i^{\mathrm{hid}}),
\end{align*}
and the observation logit is
\[
 \ell_i^{\mathrm{SP}}=0.35+0.90\sin Z_{i1}+0.60Z_{i2}Z_{i3}
 -0.35\tanh Z_{i4}+0.35\tanh A_i^r
 -1.20s_i^{\mathrm{hid}}+0.20W_i\sin(Z_{i5}+Z_{i6}).
\]

\paragraph{Deep bounded-latent shared-plus-private DGP.}
This design uses the same A4 inputs but forms eight dense shared projections
and three orthogonal private directions. The background outcome augments the
preceding shared-plus-private structure with nested sine, product, and
hyperbolic-tangent terms. The response logit uses matching dense nonlinear
projections. Let
\[
 V_i^{\mathrm{deep}}=\operatorname{sign}\{\sin(3s_i^{\mathrm{hid}})
 +.70\epsilon_i^{V}\}.
\]
The fixed background outcome additionally contains
$V_i^{\mathrm{deep}}+1.35W_iV_i^{\mathrm{deep}}$. The latent sign is never
supplied as a feature; it is only imperfectly proxied by the surrogate PCs.
This creates an outcome-bridge challenge while preserving MAR because the
response mechanism remains a function of the observed $(W,X,S)$.

Now we consider the synthetic outcome. Its empirical standard deviation
below is computed only over participants for whom this outcome is observed.
For $d\in\{\mathrm{SP},\mathrm{Deep}\}$, we scale
$\operatorname{sd}(f_i^d)=0.5\operatorname{sd}(Y_i^{\mathrm{A4}})$ and draw
\[
 Y_i=f_i^d+\tau W_i+\sigma_\varepsilon\varepsilon_i,
 \qquad \varepsilon_i\overset{\mathrm{iid}}{\sim}N(0,1),
 \qquad
 \tau_{\mathrm{true}}^d=\tau+\bar f_1^d-\bar f_0^d.
\]
The last expression is the exact fixed-A4 contrast and accounts for the
realized imbalance in $f_i^d$ across the randomized arms. Although A4
treatment was randomized, every learner estimates the treatment propensity
from $X$; no estimator is supplied $e=0.5$. This convention keeps the
implementation uniform across designs and retains treatment prediction as an
auxiliary representation-learning task. 

\section{Theoretical analysis}
\label{app:theory}
\subsection{Assumptions}
\label{app:assumptions}
In implementation, each SIRNET nuisance head is a feedforward ReLU sub-network attached to a
shared encoder output: $\hat e$ and $\hat\mu(w,\cdot)$ are functions of
$B_i=b_\theta(X_i)\in\R^{d_B}$ alone; $\hat r(w,\cdot)$ and
$\hat{\tilde\mu}(w,\cdot)$ are functions of $(B_i,H_i)$ with
$H_i=g_\theta(B_i,S_i)\in\R^{d_H}$. Writing $d_g$ for the dimension of the
representation feeding head $g$ ($d_g=d_B$ for $e,\mu$; $d_g=d_B+d_H$
for $r,\tilde\mu$), the learned-encoder rate depends on both $d_g$ and
the approximation costs of the encoders.
The following assumption describes the component structure used below.

\begin{assumption}[Representation sufficiency]
\label{ass:representation}
Suppose $X\in[-1,1]^{D_X}$ and $S\in[-1,1]^{D_S}$ almost surely.
There exist representations $B_i^0=b_0(X_i)$ and $H_i^0=g_0(B_i^0,S_i)$ of
dimensions $d_B,d_H$, and functions $\bar e_0,\bar\mu_0(w,\cdot),
\bar r_0(w,\cdot),\bar{\tilde\mu}_0(w,\cdot)$ of those representations alone, such that
\[
  e_0(X_i)=\bar e_0(B_i^0),
  \quad
  \mu_0(w,X_i)=\bar\mu_0(w,B_i^0),
  \quad
  r_0(w,X_i,S_i)=\bar r_0(w,B_i^0,H_i^0),
  \quad
  \tilde\mu_0(w,X_i,S_i)=\bar{\tilde\mu}_0(w,B_i^0,H_i^0)
\]
almost surely, for $w=0,1$. To lighten notation we write $e_0,\mu_0,r_0,\tilde\mu_0$
for $\bar e_0,\bar\mu_0,\bar r_0,\bar{\tilde\mu}_0$ from here on, and $B_i,H_i$ for
$B_i^0,H_i^0$.
The coordinates of $b_0$ and $g_0$ take values in $[-1,1]$ and belong,
respectively, to bounded balls in
$W^{p_b,\infty}([-1,1]^{D_X})$ and
$W^{p_h,\infty}([-1,1]^{d_B+D_S})$, for fixed integers $p_b,p_h\ge1$.
\end{assumption}

Assumption~\ref{ass:representation} says that the nuisance-relevant
information in the high-dimensional inputs has an intrinsic latent-dimensional 
description. This can be plausible for structured objects such as images and
text, whose ambient coordinates are numerous while task-relevant variation may
be lower-dimensional.  It is nevertheless a substantive approximation
assumption, not a consequence of the architecture, the multitask loss, or a
pretrained embedding.  
The theory is developed on this assumption, as is common in deep-network inference analyses
\citep{FarrellLiangMisra2021DNN,SchmidtHieber2020DNN}. To formalize the approximation properties 
of the nuisance functions on these latent representations, we impose a Sobolev smoothness condition 
that ensures they can be well-approximated by neural networks.

\begin{assumption}[Sobolev smoothness and minimal-rate condition]
\label{ass:smoothness}
For each
\[
 g\in\{\tilde\mu(1,\cdot),\tilde\mu(0,\cdot),
          \mu(1,\cdot),\mu(0,\cdot),e,r(1,\cdot),r(0,\cdot)\},
\]
the true nuisance $g_0$ belongs to
$W^{p_g,\infty}([-1,1]^{d_g})$, where $p_g$ is a positive integer.  More
precisely, writing $\alpha=(\alpha_1,\ldots,\alpha_{d_g})$ for a
multi-index, $|\alpha|=\sum_{j=1}^{d_g}\alpha_j$, and $D^\alpha g_0$ for
the corresponding weak derivative, this means that $D^\alpha g_0$ exists
and is essentially bounded on $[-1,1]^{d_g}$ for every $|\alpha|\le p_g$,
with Sobolev norm
\[
 \|g_0\|_{W^{p_g,\infty}}
 :=\max_{|\alpha|\le p_g}\|D^\alpha g_0\|_{L^\infty([-1,1]^{d_g})}
 \le K
\]
for a finite constant $K$, and
\[
  p_g>d_g/2.
\]
Here $d_e=d_{\mu}=d_B$ and
$d_r=d_{\tilde\mu}=d_B+d_H$.   
\end{assumption}

Assumption~\ref{ass:smoothness} is imposed on the nuisance functions after
their representation coordinates have been rescaled to $[-1,1]^{d_g}$; this
rescaling is without loss of generality for bounded representations.  It says
that these functions have uniformly bounded derivatives through order $p_g$,
thereby ruling out arbitrarily sharp variation in the low-dimensional
coordinates.  The threshold $p_g>d_g/2$ is exactly the minimal condition for
the usual nonparametric exponent to satisfy
\[
  \frac{p_g}{2p_g+d_g}>\frac14,
\]
but the learned-encoder theorem additionally requires
$p_b/D_X>1/2$ and $p_h/(d_B+D_S)>1/2$ to control approximation of
the two encoders. These are sufficient smoothness conditions, not a claim
that the nuisances are literally Sobolev functions in the original
high-dimensional input space. Their burden grows with the ambient
encoder-input dimensions as well as the representation dimensions.

\subsection{Theoretical Properties}
\label{sec:theory}

In this section, we discuss the theoretical properties of SIRNET. 
The first result is a general cross-fitted oracle inference theorem by \citet[Theorems~2.1, 3.2, and~3.3]{KallusMao2024LimitedOutcomeData}, restated in our
notation, for the sake of self-containment.  It is generally learner-agnostic and applicable on any method that produces cross-fitted nuisance
estimators meeting consistency and product-rate conditions. For a cross-fitting fold $k$, write
$P_w$ for the distribution conditional on $W=w$. Define
$\rho_{e,k}=\|\hat e_k-e_0\|_{L^2(P)}$ and
$\rho_{\mu(w),k}=\|\hat\mu_k(w,\cdot)-\mu_0(w,\cdot)\|_{L^2(P)}$,
where $P$ is the marginal distribution of $X$ for these $X$-only
functions. For $g\in\{r,\tilde\mu\}$, define
$\rho_{g(w),k}=\|\hat g_k(w,\cdot)-g_0(w,\cdot)\|_{L^2(P_w)}$
on $(X,S)$. Let $\rho_g$ denote the maximum over the fixed number of
folds and, where applicable, treatment arms. These arm-specific norms
measure errors on the surrogate distribution actually used in each arm's
score, without requiring overlap between the post-treatment surrogate
distributions across arms. Equivalently, each arm-specific norm is the
$L^2(P)$ norm of the corresponding error multiplied by $\ind{W=w}$,
divided by $\Pr(W=w)^{1/2}$.

\begin{theorem}[Kallus--Mao cross-fitted oracle inference]
\label{thm:km-oracle}
Let $O_1,\ldots,O_n$ be i.i.d., 
and suppose Assumptions~\ref{ass:ignorability} and
\ref{ass:missing-at-random} hold. For a fixed number of cross-fitting
folds, suppose each evaluation fold is independent of its nuisance fits.
For some fixed $\epsilon_0>0$, require almost surely
$\hat e_k(X)\in[\epsilon_0,1-\epsilon_0]$ and
$\hat r_k(w,X,S)\in[\epsilon_0,1]$ for both arms and every fold.
Suppose every fold-specific nuisance
estimate satisfies
\[
 \rho_e,\ \rho_{\mu(w)},\ \rho_{r(w)},\
 \ \rho_{\tilde\mu(w)}=o_P(1),\qquad w\in\{0,1\}.
\]
Assume that for both arms
\begin{equation}
 \max_{w\in\{0,1\}}
 \left\{
   \rho_e\rho_{\mu(w)},\quad
   \rho_e\rho_{\tilde\mu(w)},\quad
   \rho_{r(w)}\rho_{\tilde\mu(w)}
 \right\}=o_P(n^{-1/2}),
 \label{eq:product-rate}
\end{equation}
and the bounded-moment conditions of
\citet[Assumption~6]{KallusMao2024LimitedOutcomeData}: for some
$C<\infty$ and $q>2$, the conditional variances of
$Y$ given $(W,X,S)$ and given $(W,X)$, and of
$\tilde\mu_0(W,X,S)$ given $(W,X)$, are essentially bounded by $C$,
while $\max_w\E|Y(w)|^q\le C$. Then the cross-fitted estimator in
equation~\eqref{eq:sirnet-estimator} satisfies
\[
 \sqrt n\{\hat\tau-\tau_0\}\ \rightsquigarrow\ N(0,V^*),
 \qquad
 n\,\widehat{\mathrm{se}}(\hat\tau)^2\ \to_P\ V^*,
\]
where $V^*$ is the efficiency bound in
equation~\eqref{eq:efficiency-bound}. If $V^*>0$, the confidence interval
based on equation~\eqref{eq:sirnet-se} has asymptotic coverage $1-\alpha$.
\end{theorem}

To apply Theorem~\ref{thm:km-oracle}, we need to ensure that the cross-fitted
nuisance estimators satisfy its consistency and product-rate conditions.
We now introduce notation for the empirical optimization gaps used in a
conditional learned-encoder rate result. For each training fold, let
$\Theta_n$ be a growing joint ReLU sieve with trainable
$b_\theta,g_\theta$, and let $\hat\theta\in\Theta_n$ be the final fit.
The supervised part of~\eqref{eq:training-loss} is
\[
 \widehat{\mathcal R}_{\mathrm{sup}}(\theta)
 =\lambda_Y P_{n,R=1}\{Y-\tilde\mu_\theta(W,X,S)\}^2
 +\lambda_R P_n\ell_{\mathrm{ce}}\{R,r_\theta(W,X,S)\}
 +\lambda_W P_n\ell_{\mathrm{ce}}\{W,e_\theta(X)\},
\]
where $P_{n,R=1}$ averages over labeled units. Define
\begin{equation}
 \delta_{\mathrm{sup},n}
 :=\widehat{\mathcal R}_{\mathrm{sup}}(\hat\theta)
   -\inf_{\theta\in\Theta_n}
       \widehat{\mathcal R}_{\mathrm{sup}}(\theta).
 \label{eq:joint-supervised-gap}
\end{equation}
For each $w$, let $\hat q_w(X,S)=\hat{\tilde\mu}(w,X,S)$ and
$\hat f_w(X)=\hat\mu(w,X)$ denote the final fitted predictions on the
original inputs. With $\hat q_w$ held fixed, define
\begin{equation}
 \delta_{\mu(w),n}
 :=P_{n,w}(\hat q_w-\hat f_w)^2
   -\inf_{f\in\mathcal F_{m,w}}P_{n,w}(\hat q_w-f)^2,
 \label{eq:marginal-block-erm}
\end{equation}
where $P_{n,w}$ averages over $W_i=w$ and $\mathcal F_{m,w}$ contains
all bounded, clipped marginal composites $x\mapsto h_w(b(x))$ in the
joint sieve, allowing the comparator baseline encoder $b$ to vary.
These gaps are nonnegative when the fitted outputs belong to their
respective classes; defining them does not assert that they are small.
In particular, minimizing the full objective~\eqref{eq:training-loss}
with its generated marginal target and stop-gradient updates does not
guarantee a rate for either gap. Instead, the gaps are a way to state the optimization requirement for the learned-encoder nuisance rates, which is separate from the statistical
argument. They also appears in other ways in the literature on deep learning theory, such as in the analyses of deep regression \citep{SchmidtHieber2020DNN} and empirical minimizers \citep{FarrellLiangMisra2021DNN}. The resulting rate theorem applies to
final fits when the order of the optimization gaps shrinks in a certain magnitude as specified in the next theorem. Nevertheless, establishing them
for a particular training algorithm is a separate question.

The following theorem gives conditional
sufficient rates for the learned-encoder nuisances.
\begin{theorem}[Conditional learned-encoder composite nuisance rates]
\label{thm:composite-direct-rate}
Suppose Assumptions~\ref{ass:ignorability}, \ref{ass:missing-at-random},
\ref{ass:representation}, and~\ref{ass:smoothness} hold. On each of a
fixed number of training folds, suppose the fitted heads belong to the
joint sieve described above, with $\hat f_w\in\mathcal F_{m,w}$, and
the supervised weights $\lambda_Y,\lambda_R,\lambda_W$ are fixed and
bounded above and away from zero.
Assume outcome and marginal heads are clipped to bounded intervals
containing their true targets, and the probability heads $e$ and $r(w)$
are clipped to $[\epsilon_0,1-\epsilon_0]$ for some fixed
$0<\epsilon_0\le\epsilon$. Assume uniformly conditionally sub-Gaussian outcome residuals on the
$W=w,R=1$ sub-populations for both arms.
Set
\[
 s_{\mathrm{comp}}=\min\left\{\frac{p_b}{D_X},
 \frac{p_h}{d_B+D_S},\min_g\frac{p_g}{d_g}\right\},
\]
where the last minimum is over the seven heads of
Assumption~\ref{ass:smoothness}. Suppose $s_{\mathrm{comp}}>1/2$ and we use
a growing joint ReLU sieve with a fixed layered architecture at each $n$,
hidden-layer widths $O(m_n)$, depth $O(\log m_n)$, and
$O(m_n\log m_n)$ total nonzero parameters of uniformly bounded magnitude
across the shared encoders and seven heads, where
$m_n\asymp\{n/(\log n)^3\}^{1/(2s_{\mathrm{comp}}+1)}$. Write 
$\mathcal C_n=m_n(\log n)^3$.
\begin{itemize}
  \item Outcome, propensity, and reponse heads error. For the five directly supervised
heads $\mathcal G_{\mathrm{sup}}=\{e,r(0),r(1),\tilde\mu(0),\tilde\mu(1)\}$,
the fitted composite nuisances obey
\[
 \sum_{g\in\mathcal G_{\mathrm{sup}}}\rho_g^2
 =O_P\!\left(m_n^{-2s_{\mathrm{comp}}}
    +\frac{\mathcal C_n}{n}+\delta_{\mathrm{sup},n}\right),
\]
  \item Marginalization heads error. For each $w\in\{0,1\}$,
\[
 \rho_{\mu(w)}^2
 =O_P\!\left(m_n^{-2s_{\mathrm{comp}}}
    +\frac{\mathcal C_n}{n}+\delta_{\mathrm{sup},n}
    +\delta_{\mu(w),n}\right).
\]
\end{itemize}
Here the gaps and errors denote their maxima across folds. In particular,
if $\delta_{\mathrm{sup},n}=o_P(n^{-1/2})$ and
$\max_w\delta_{\mu(w),n}=o_P(n^{-1/2})$, all seven nuisance errors
are $o_P(n^{-1/4})$. Under the additional statistical conditions of
Theorem~\ref{thm:km-oracle}, its inference conclusions then apply to
the learned-encoder score estimator.
\end{theorem}

Theorem~\ref{thm:composite-direct-rate} separates approximation,
statistical complexity, and the two kinds of optimization gap. The 
smoothness threshold $s_{\mathrm{comp}}>1/2$ constrains both learned
encoders and nuisance heads. As long as we train the model to satisfy the optimization conditions, the estimator is asymptotically normal and efficient. Nevertheless, as we have emphasized, the theorem is conditional on the stated gap rates.  It does not establish those rates for ordinary gradient-based SIRNET training, which is a separate question.

In particular, a key component that we have adopted in the training process is the stop-gradient operation on the generated targets. Our marginalization objective can be interpreted as a conditional-expectation projection with generated targets. For each treatment arm, the observed-outcome regression supplies detached pseudo-outcomes, which are used to learn their conditional expectation given baseline covariates. This directional dependence is analogous to detached-target recursive learning in DeepBlip \citep{ma2025deepblip}. At the population level, the marginal regression error is bounded by the outcome-regression error plus the marginal projection error. The former is established in the first part of Theorem \ref{thm:composite-direct-rate} for the outcome head error term. Our same-fold generated-regressor lemma (Lemma \ref{lem:same-fold-marginal}) controls the latter under explicit approximation, complexity, and optimization conditions. Stop-gradient prevents direct differentiation through the generated targets, but does not prevent marginalization-loss updates to the shared encoder from affecting subsequent outcome predictions. Accordingly, our theoretical results provide conditional statistical guarantees rather than a convergence guarantee for the joint gradient-based training dynamics.

\section{More on A4 study}
\subsection{A4 study acknowledgement statement}
The A4 Study is a secondary prevention trial in preclinical Alzheimer's disease, aiming to slow cognitive decline associated with brain amyloid accumulation in clinically normal older individuals. The A4 Study is funded by a public-private-philanthropic partnership, including funding from the National Institutes of Health-National Institute on Aging, Eli Lilly and Company, Alzheimer's Association, Accelerating Medicines Partnership, GHR Foundation, an anonymous foundation and additional private donors, with in-kind support from Avid and Cogstate. The companion observational Longitudinal Evaluation of Amyloid Risk and Neurodegeneration (LEARN) Study is funded by the Alzheimer's Association and GHR Foundation. The A4 and LEARN Studies are led by Dr. Reisa Sperling at Brigham and Women's Hospital, Harvard Medical School and Dr. Paul Aisen at the Alzheimer's Therapeutic Research Institute (ATRI), University of Southern California. The A4 and LEARN Studies are coordinated by ATRI at the University of Southern California, and the data are made available through the Laboratory for Neuro Imaging at the University of Southern California. The participants screening for the A4 Study provided permission to share their de-identified data in order to advance the quest to find a successful treatment for Alzheimer's disease. We would like to acknowledge the dedication of all the participants, the site personnel, and all of the partnership team members who continue to make the A4 and LEARN Studies possible. The complete A4 Study Team list is available on: \url{www.actcinfo.org/a4-study-team-lists}.

\subsection{A4 Overlap and Missingness Sensitivity Diagnostics}
\label{app:a4-diagnostics}

The conditional MAR assumption cannot be tested from the observed data.  We
therefore report diagnostics that check its empirical implications and the
stability of the weighting correction. 

\paragraph{Missingness reasons and death.}
Table~\ref{tab:a4-disposition} breaks down the 332 participants with missing
month~66 WMH by their recorded reason for study discontinuation. Subject withdrawal (172) and adverse events (51) account for
most of the missingness; 11 participants (6 solanezumab, 5 placebo) have a
recorded death before month~66. Because $Y_i(w)=\log(\mathrm{WMH}_i(w)+1)$ is
not well defined for a participant who dies before the month~66 visit under
arm $w$, treating these 11 cases as ordinary missing-at-random nonresponse is less appropriate. Nevertheless, with only 11 of 1129 participants (1.0\%)
affected, this choice is unlikely to materially change the point estimate, but
we flag it explicitly rather than leave death implicitly folded into
``dropout.''

\begin{table}[h]
\centering
\caption{Reason for study discontinuation among
  the 332 A4 analysis participants missing month~66 WMH.}
\label{tab:a4-disposition}
\small
\begin{tabular}{@{}lr@{}}
\toprule
Reason & $n$ \\
\midrule
Withdrawal by subject         & 172 \\
Adverse event                 & 51 \\
Completed (visit not reached) & 31 \\
Other                         & 26 \\
No disposition record         & 13 \\
Death                         & 11 \\
Safety risk                   & 7 \\
Lack of efficacy              & 7 \\
Lost to follow-up             & 6 \\
Physician decision            & 5 \\
Withdrawal by parent/guardian & 2 \\
Protocol deviation            & 1 \\
\midrule
Total                         & 332 \\
\bottomrule
\end{tabular}
\end{table}

The diagnostic refit uses one five-fold cross-fit of the selected architecture
and the $n=1129$ analysis sample.  Outcome observation rates are nearly equal
in the control and treatment arms (70.5\% and 70.7\%, respectively).  As shown
in Figure~\ref{fig:a4-overlap}, the estimated treatment propensity has broad
support (1st--99th percentile: 0.175--0.841).  The observed-arm labeling
propensity has a median of 0.740, although its minimum is 0.012.  Consequently,
the observed-case weights have a long right tail: the 99th percentile is 15.0
and the maximum is 104.0.  Therefore, we adopted proper trimming in the implementation to mitigate the outlier weights and keep the estimator stable. Table~\ref{tab:a4-rclip-sensitivity} shows that, over labeling-propensity
floors from 0.05 to 0.15, the point estimate ranges from $-0.013$ to $-0.008$
and every interval includes zero.  Lowering the floor to 0.01 increases the
standard error and moves the point estimate to $-0.001$, illustrating the
influence of the extreme tail without changing the qualitative conclusion.

\begin{figure*}[t]
  \centering
  \includegraphics[width=\textwidth]{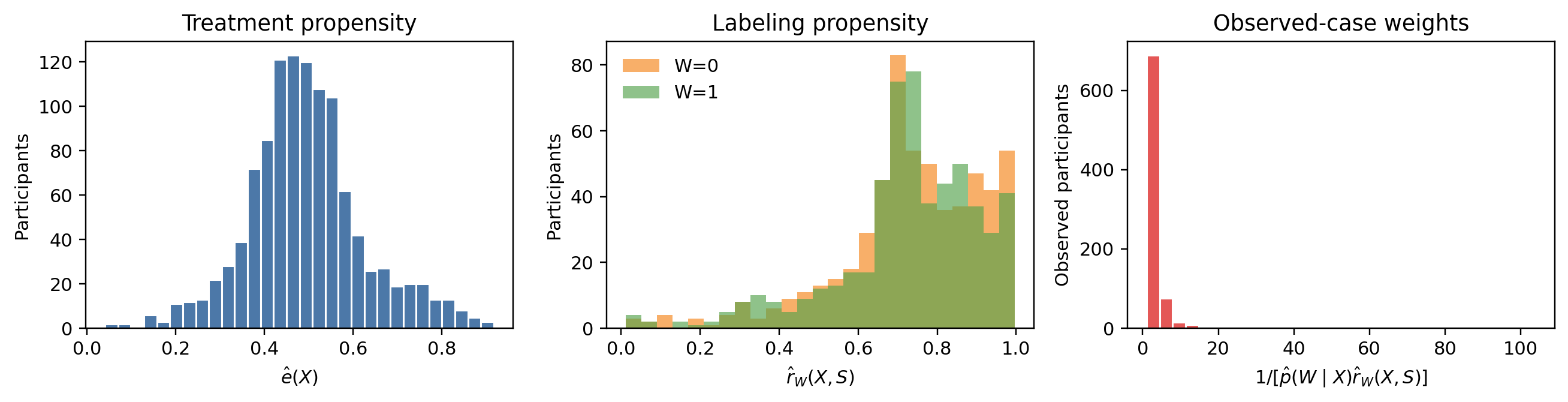}
  \caption{Out-of-fold overlap diagnostics from the five-fold A4 refit.
  Left: estimated treatment propensity.  Centre: estimated probability that
  the month~66 outcome is observed under the received arm.  Right:
  observed-case treatment--labeling weights.}
  \label{fig:a4-overlap}
\end{figure*}

\begin{table}[htbp!]
\centering
\caption{Sensitivity of a five-fold SIRNET diagnostic refit to the floor
  imposed on the estimated labeling propensity.}
\label{tab:a4-rclip-sensitivity}
\begin{tabular}{@{}rrrr@{}}
\toprule
$r_{\mathrm{clip}}$ & $\hat\tau$ & SE & 95\% CI \\
\midrule
0.01 & $-0.001$ & $0.045$ & $[-0.090,\;+0.087]$ \\
0.05 & $-0.013$ & $0.041$ & $[-0.093,\;+0.068]$ \\
0.10 & $-0.011$ & $0.038$ & $[-0.086,\;+0.063]$ \\
0.15 & $-0.008$ & $0.036$ & $[-0.079,\;+0.064]$ \\
\bottomrule
\end{tabular}
\end{table}

\section{Computational aspects}
\label{app:computational-aspects}

\paragraph{Time.}
Let $n$ be the number of observations, $P$ the number of trainable SIRNET
parameters, $E$ the epochs actually run (bounded by the configured maximum),
and $K$ the number of cross-fitting folds. For the dense feedforward network,
one forward/backward pass costs $O(P)$ per observation, up to the fixed
number of nuisance heads. Adam training on one fold therefore costs
$O(EnP)$ in the worst case; held-out prediction and score aggregation cost
$O(nP)$ and $O(n)$, respectively. A standard $K$-fold round fits $K$
networks on approximately $n(K-1)/K$ observations each, giving
$O(E(K-1)nP)$ training time, plus fold-local preprocessing. Repeating the
round $R$ times multiplies this cost by $R$. For nested tuning with $C$
candidate configurations and $J$ inner folds, each outer fold fits $CJ$
inner models and refits one selected model. The implementation thus trains
$RK(CJ+1)$ networks; for the A4 full-score analysis, $R=5$, $K=5$,
$C=3$, and $J=2$ give 175 fits. These counts are an upper-bound guide to
workload, not measured wall-clock times: early stopping, candidate widths,
PCA, and hardware affect actual runtime. Each split fits its scalers and,
for A4, the baseline and surrogate PCA maps on training data only.

\paragraph{Space and sample size.}
The fitted network and Adam's parameter and moment arrays require $O(P)$
storage, while minibatch activations scale with batch size and layer widths.
Holding the preprocessed sample and out-of-fold scores requires space linear
in $n$ and the input dimensions; the A4 pipeline also holds the raw imaging
embeddings and PCA work arrays during preprocessing. Models are fitted one
split at a time rather than keeping all $RK(CJ+1)$ models in memory. An
outer fit uses approximately $n(K-1)/K$ units before its early-stopping
split; an inner candidate fit uses approximately
$n(K-1)(J-1)/(KJ)$ units before that split. With $n=1129$, $K=5$,
$J=2$, and a 15\% validation fraction, these are roughly 770 and 380
training units, respectively. Outcome-head training uses only observed
outcomes in the relevant treatment arm, so its effective sample is smaller
still (the A4 outcome is observed for 797 of 1129 analysis participants).

\paragraph{Computing infrastructure.}
The experiment server documented for this project has one NVIDIA H100 GPU
with 80\,GB of memory. Its workspace allocation is limited to approximately
14 CPU cores and 146\,GiB of RAM. The neural simulation launch scripts
request CUDA for SIRNET fits, whereas classical and XGBoost baseline jobs
use bounded CPU parallelism. A4 SIRNET runs also
request CUDA; the training code falls back to CPU if a GPU is unavailable.
GPU access depends on the execution environment: the managed sandbox does
not expose the host GPU, so a CUDA job must run in a GPU-enabled allocation.
The recorded configuration describes the project server, not a per-run
hardware log or a measured runtime for each reported experiment.

\section{Hyperparameter Tuning}
\label{app:hparam}

\subsection{SIRNET}

SIRNET has architecture parameters that control the shared baseline and
surrogate embeddings, and loss weights
$(\lambda_Y,\lambda_R,\lambda_W,\lambda_M)$ that balance the four jointly
trained heads.  

\paragraph{Nested tuning for the full score.}
The full-score A4 analysis uses nested tuning rather than a global validation
screen.  For each outer $K=5$ training partition, we evaluate the three fixed
candidate configurations in Table~\ref{tab:hparam-search} by two-fold inner
cross-fitting.  The inner criterion is the standard deviation of the
out-of-fold per-unit orthogonal scores.  We then refit the selected
configuration on the full outer training partition (with a training-only
early-stopping split) and evaluate it only on the untouched outer fold.  Thus a
participant's outcome and influence contribution cannot affect the
configuration used for that participant's final score.  The procedure is
repeated for five outer $K=5$ rounds.  Early stopping uses at most 300 epochs
with patience 40, and all configurations use the response-probability score
floor $r_{\mathrm{clip}}=0.05$.

The compact regularized candidate is selected in 22 of the 25 outer folds;
the balanced and outcome-bridge candidates are selected in 2 and 1 folds,
respectively.  The surrogate-only analysis remains a separately reported
bootstrap-based descriptive estimator and does not supply the full-score
orthogonal-inference result.

\begin{table*}[h]
\centering
\caption{Fixed candidate configurations for nested tuning of the A4
  full SIRNET score.  Each outer fold is scored only after selecting one row
  using inner-fold predictions from its outer training partition.  The final
  column reports selections across five repeated $K=5$ outer splits.}
\label{tab:hparam-search}
\scriptsize
\setlength{\tabcolsep}{3pt}
\resizebox{\textwidth}{!}{%
\begin{tabular}{@{}lrrrrrrrr@{}}
\toprule
Candidate & $d_S$ & $d_B$ & $(d_b,d_h)$ & Dropout & $\eta$ & $\lambda_{\mathrm{wd}}$ & $(\lambda_Y,\lambda_R,\lambda_W,\lambda_M)$ & Selected \\
\midrule
Balanced                   & 96 & 96 & $(64,128)$ & 0.30 & $10^{-3}$ & $10^{-4}$ & $(5,.5,.25,1)$ & 2 \\
Outcome bridge             & 96 & 96 & $(64,128)$ & 0.20 & $10^{-3}$ & $5\!\times\!10^{-4}$ & $(10,.5,.25,2)$ & 1 \\
Compact regularized        & 64 & 64 & $(32,64)$  & 0.30 & $10^{-3}$ & $10^{-3}$ & $(5,.5,.25,1)$ & 22 \\
\bottomrule
\end{tabular}
}
\end{table*}

\paragraph{Final seed averaging and inference.} To reduce variance from randomness in neural network training, we employ a five-seed averaging strategy.
For full SIRNET, we average the out-of-fold influence values over five
independently seeded, nested $K=5$ cross-fitting rounds before computing the
ATE and its influence-function SE. For Surr-Only [SIRNET], which uses
bootstrap inference, the point estimate is the mean of five independently
initialized shared-network fits. Each of the $B=200$ paired bootstrap draws
repeats the same five-seed averaging, so its bootstrap SE incorporates both
sampling variation and the observed optimization sensitivity.

\subsection{Baseline nuisance estimators and XGBoost selection}

Ridge baselines use fixed configurations.  For the direct DML-XS--SIRNET
comparisons in Tables~\ref{tab:mc-main} and~\ref{tab:semisynthetic}, however,
we tune the full XGBoost nuisance pipeline with the same split discipline as
the causal estimator: profile selection occurs inside each outer training
partition, and the selected pipeline is refit before scoring the untouched
outer fold.  

\paragraph{Ridge (linear nuisances).}
Outcome models $\mu_w(X)$ and $\tilde\mu_w(X,S)$ use
\texttt{sklearn} \texttt{Ridge}($\alpha=1.0$) with $z$-scored inputs.
Selection models $r_w(X)$ and $r_w(X,S)$ use \texttt{LogisticRegression}
with $\ell_2$ penalty ($C=1.0$) and standardised inputs.

\paragraph{XGBoost.}
For the nested DML-XS rows, each candidate specifies the complete nuisance
pipeline: treatment and labeling classifiers, observed-outcome regressions,
and the generated marginal regressions all use the same profile.  In every
outer fold, we compare the three profiles in Table~\ref{tab:xgb-profiles} by
the standard deviation of their inner out-of-fold orthogonal scores, then
refit the winner on the outer training data.  %

\begin{table}[h]
\centering
\caption{Fixed XGBoost candidate profiles used for fold-local selection in
the Table~\ref{tab:mc-main} and Table~\ref{tab:semisynthetic} DML-XS rows.
All profiles use \texttt{subsample}=\texttt{colsample\_bytree}=0.8 and
\texttt{n\_jobs}=1.}
\label{tab:xgb-profiles}
\small
\begin{tabular}{@{}lrrrl@{}}
\toprule
Profile & Trees & Depth & Learning rate & Additional setting \\
\midrule
Shallow & 100 & 3 & 0.10 & --- \\
Moderate & 300 & 4 & 0.05 & --- \\
Deep & 500 & 5 & 0.03 & \texttt{min\_child\_weight}=3 \\
\bottomrule
\end{tabular}
\end{table}


\paragraph{Bootstrap standard errors.}
Imputation-only methods (RI-X, RI-XS, pooled surrogate-prediction contrast, Surr-Only) use
$B=30$ bootstrap resamples for variance estimation in the MC and semisynthetic
experiments, and $B=200$ in the A4 empirical application.  The simulation
choice is computational: every bootstrap draw refits the full $K$-fold nuisance
pipeline, nested inside 300 Monte Carlo replications, two signal levels, two
DGPs, and two learner families.  The point estimates, empirical SD, and RMSE do
not depend on $B$; only the replicate-specific bootstrap interval does. 




\section{Additional Technical Results and Proof}

\subsection{Robustness of the efficient score}

In this section, we establish the robustness properties of the efficient score under various nuisance function misspecifications.

\begin{proposition}[Robustness of the full score]
\label{prop:score-robustness}
Under Assumptions~\ref{ass:ignorability} and~\ref{ass:missing-at-random},
let $(\tilde\mu^\star,\mu^\star,e^\star,r^\star)$ be probability limits of
the cross-fitted nuisances, with $e^\star(X)$, $1-e^\star(X)$, and
$r^\star(w,X,S)$ bounded away from zero for both arms.
Assume the expectations below are finite.
For $w\in\{0,1\}$ define
\[
 e_{0w}(X)=\Pr(W=w\mid X),\qquad
 e^\star_w(X)=w e^\star(X)+(1-w)\{1-e^\star(X)\},
\]
and let
\[
 \mu_{0w}(X)=\E\{\tilde\mu_0(w,X,S)\mid W=w,X\}.
\]
The population bias of the score at the true ATE is
\begin{align}
 &\E\{\psi(O;\tau_0,\bfeta^\star)\}\nonumber\\
 &=\sum_{w=0}^1(2w-1)\E\Bigg[
 \left\{1-\frac{e_{0w}(X)}{e^\star_w(X)}\right\}
 \{\mu^\star(w,X)-\mu_{0w}(X)\}\nonumber\\
 &+\frac{e_{0w}(X)}{e^\star_w(X)}
 \E\left[\left.
 \left\{\frac{r_0(w,X,S)}{r^\star(w,X,S)}-1\right\}
 \{\tilde\mu_0(w,X,S)-\tilde\mu^\star(w,X,S)\}
 \right|W=w,X\right]\Bigg].
 \label{eq:score-bias}
\end{align}
Consequently, a sufficient condition for consistency is that, in each arm,
\[
 \{e^\star=e_0\ \text{or}\ \mu^\star(w,\cdot)=\mu_{0w}\}
 \quad\text{and}\quad
 \{r^\star(w,\cdot)=r_0(w,\cdot)\ \text{or}\
   \tilde\mu^\star(w,\cdot)=\tilde\mu_0(w,\cdot)\}.
\]

\end{proposition}

Proposition~\ref{prop:score-robustness} shows that the efficient score is robust to certain types of nuisance function misspecifications. In particular, it remains consistent as long as either the propensity score or the outcome regression is correctly specified for each treatment arm, and either the missingness mechanism or the imputed outcome is correctly specified for each arm. 

\subsection{Collection of technical lemmas and proofs}

\paragraph{Approximation and complexity for the learned joint sieve.}
Yarotsky's approximation result (Lemma~\ref{lem:yarotsky}) explains
the Sobolev exponent; the bounded-weight construction of
\citet[Theorem~5]{SchmidtHieber2020DNN} supplies approximants for the
actual sieve. Lemma~\ref{lem:covering-number} controls complexity, and
Lemma~\ref{lem:oracle-ineq} supplies the direct-head risk bound.

\begin{lemma}[{ReLU approximation of Sobolev balls; \citealp[Theorem~1]{Yarotsky2017DNN}}]
\label{lem:yarotsky}
Fix $d\ge1$, integer $p\ge1$, and $0<K<\infty$. There is a constant
$C_{d,p,K}<\infty$ such that for every $f_0\in W^{p,\infty}([-1,1]^d)$ with
$\|f_0\|_{W^{p,\infty}}\le K$ and every $\epsilon\in(0,1)$, there exists a
feedforward ReLU network $f$ with depth at most
$C_{d,p,K}(\log(1/\epsilon)+1)$ and at most
$C_{d,p,K}\epsilon^{-d/p}(\log(1/\epsilon)+1)$ weights and computation
units, such that $\sup_{x\in[-1,1]^d}|f(x)-f_0(x)|\le\epsilon$.
\end{lemma}

Lemma~\ref{lem:yarotsky} is a statement purely about what a ReLU network
can represent; it involves no data, no randomness, and no estimation. To account for
the statistical error of estimating a function from data, we need to bound the
covering number of the class of networks of a given size. The following lemma
gives a covering-number bound for a class of ReLU networks with bounded depth,
width, and parameter magnitude. 

For the bounded-weight sieve used below we use the constructive approximation
bound of \citet[Theorem~5]{SchmidtHieber2020DNN}, rather than infer a weight
bound from Lemma~\ref{lem:yarotsky}. Applied to both trainable encoders and
all seven heads on their respective bounded input domains, it yields the
joint composite approximation and entropy bounds in the proof of
Theorem~\ref{thm:composite-direct-rate}.

\begin{lemma}[Covering number of a ReLU network class]
\label{lem:covering-number}
For integers $L,W\ge1$ and $B\ge1$, let $\mathcal F(L,W,B)$ denote the
functions $[-1,1]^d\to\R$ computed by a feedforward ReLU network of a fixed
architecture with $L$ hidden layers and $W$ total (weight and bias)
parameters, each parameter ranging over $[-B,B]$. Then for every
$\delta\in(0,1]$,
\[
  \log\mathcal N\bigl(\delta,\mathcal F(L,W,B),\|\cdot\|_\infty\bigr)
  \;\le\; W(L+4)\log\bigl(B(W+1)\bigr) + W\log(4B/\delta).
\]
If $\mathcal F_s(L,W,B)$ restricts this architecture to at most
$s\in\{1,\ldots,W\}$ nonzero parameters, then
\[
 \log\mathcal N\bigl(\delta,\mathcal F_s(L,W,B),\|\cdot\|_\infty\bigr)
 \le s\log(eW/s)+s(L+4)\log\bigl(B(W+1)\bigr)
      +s\log(4B/\delta).
\]
\end{lemma}

\begin{proof}
Write $f_\theta$ for the network with parameter vector
$\theta\in[-B,B]^W$, with layer activations $a_\theta^{(0)}(x)=x$,
$a_\theta^{(\ell)}(x)=\mathrm{ReLU}\bigl(W^{(\ell)}_\theta a_\theta^{(\ell-1)}(x)+b^{(\ell)}_\theta\bigr)$
for $\ell=1,\ldots,L$, and $f_\theta(x)=W^{(L+1)}_\theta a_\theta^{(L)}(x)+b^{(L+1)}_\theta$.
Let $\Gamma:=B(W+1)\ge2$.

\emph{Layer bound.} We show $\|a_\theta^{(\ell)}(x)\|_\infty\le\Gamma^{\ell+1}$
for all $\ell=0,\ldots,L$ and $x\in[-1,1]^d$, by induction. The base case
$\ell=0$ holds since $\|x\|_\infty\le1\le\Gamma$. For the step, since the
network has at most $W$ total parameters, each unit has at most $W$ incoming
weights, so each coordinate of the pre-activation
$z^{(\ell)}=W^{(\ell)}_\theta a_\theta^{(\ell-1)}(x)+b^{(\ell)}_\theta$
satisfies $|z^{(\ell)}_k|\le BW\|a_\theta^{(\ell-1)}(x)\|_\infty+B \le
BW\Gamma^\ell+B\Gamma^\ell=\Gamma^{\ell+1}$ (using the inductive hypothesis
and $\Gamma^\ell\ge1$). Since $|\mathrm{ReLU}(t)|\le|t|$,
$\|a_\theta^{(\ell)}(x)\|_\infty\le\Gamma^{\ell+1}$, completing the induction;
the same bound gives $|f_\theta(x)|\le\Gamma^{L+2}$.

\emph{Lipschitz bound.} Fix $\theta,\theta'\in[-B,B]^W$ with
$\eta:=\|\theta-\theta'\|_\infty$, and let
$\eta_\ell:=\sup_x\|a_\theta^{(\ell)}(x)-a_{\theta'}^{(\ell)}(x)\|_\infty$, so
$\eta_0=0$. As above, each coordinate of $z^{(\ell)}-z'^{(\ell)}$ is a sum of
at most $W$ weight-difference terms, a weight-times-activation-difference
term, and a bias-difference term, so
\[
  |z^{(\ell)}_k-z'^{(\ell)}_k|
  \;\le\; BW\eta_{\ell-1} + W\eta\,\Gamma^\ell + \eta
  \;\le\; \Gamma\,\eta_{\ell-1} + \Gamma^{\ell+1}\eta,
\]
using $\|a_{\theta'}^{(\ell-1)}(x)\|_\infty\le\Gamma^\ell$ from the layer
bound applied to $\theta'$, $BW\le\Gamma$, and
$W\Gamma^\ell+1\le(W+1)\Gamma^\ell\le\Gamma^{\ell+1}$. Since ReLU is
$1$-Lipschitz, $\eta_\ell\le\Gamma\eta_{\ell-1}+\Gamma^{\ell+1}\eta$, and by
induction on $\ell$ (base case $\eta_0=0$), $\eta_\ell\le\ell\,\Gamma^{\ell+1}\eta$.
Applying the same one-step bound to the (linear) output layer,
\[
  \|f_\theta-f_{\theta'}\|_\infty
  \;\le\; \Gamma\,\eta_L + \Gamma^{L+2}\eta
  \;\le\; \Gamma\cdot L\Gamma^{L+1}\eta + \Gamma^{L+2}\eta
  \;\le\; (L+1)\Gamma^{L+2}\eta
  \;\le\; \Gamma^{L+4}\eta,
\]
using $L\le W\le\Gamma$ in the last step.

\emph{Covering argument.} By the Lipschitz bound, an $\eta$-net of the
parameter box $[-B,B]^W$ in $\|\cdot\|_\infty$, with
$\eta=\delta/\Gamma^{L+4}$, induces a $\delta$-net of
$\mathcal F(L,W,B)$ in sup norm over $x\in[-1,1]^d$. A standard grid
construction gives such an $\eta$-net of $[-B,B]^W$ of size at most
$(4B/\eta)^W$ (for $\eta\le B$), so
\[
  \mathcal N(\delta,\mathcal F(L,W,B),\|\cdot\|_\infty)
  \;\le\; \Bigl(\frac{4B\Gamma^{L+4}}{\delta}\Bigr)^{W},
\]
and taking logarithms and substituting $\Gamma=B(W+1)$ gives the first
bound. For the sparse class, fix a support of size at most $s$.
The same Lipschitz constant applies, while a parameter grid on that
support needs at most $(4B\Gamma^{L+4}/\delta)^s$ points. There are
at most $\sum_{j=0}^s\binom Wj\le(eW/s)^s$ such supports. Taking the
union of their grids and then logarithms gives the second bound.
A fixed coordinatewise clipping of the outputs does not enlarge either
covering number.
\end{proof}

\begin{lemma}[Least-squares oracle inequality for approximate ERM]
\label{lem:oracle-ineq}
Let $\mathcal F$ consist of functions $[-1,1]^d\to[-M,M]$, where
$M\ge1$, and suppose
$\log\mathcal N(\delta,\mathcal F,\|\cdot\|_\infty)
\le A\log(A_0/\delta)$ for $0<\delta\le M$, with $A\ge1$ and
$A_0\ge\max\{e,M\}$. Let $(x_i,Y_i)_{i=1}^n$ be i.i.d.,
$Y_i=f_0(x_i)+\epsilon_i$, $\|f_0\|_\infty\le M$, and suppose
$\E[\epsilon_i\mid x_i]=0$ and the errors are uniformly conditionally
sub-Gaussian with parameter $\sigma$. Suppose $\hat f\in\mathcal F$
satisfies
\[
 P_n(Y-\hat f)^2\le\inf_{f\in\mathcal F}P_n(Y-f)^2+\delta_n
\]
for a nonnegative, possibly data-dependent $\delta_n$. For $n\ge2$,
there is a universal constant $C$ such that, if $\E\delta_n<\infty$,
\[
  \E\|\hat f-f_0\|_{L^2(P)}^2
  \le C\inf_{f\in\mathcal F}\|f-f_0\|_\infty^2
    +C(M^2+\sigma^2)\frac{A\log(A_0n)}{n}
    +C\E\delta_n.
\]
\end{lemma}

This is the localized least-squares oracle-inequality argument adapted to
conditional sub-Gaussian errors; it is not a direct restatement of the
Gaussian-noise result of \citet[Lemma~4]{SchmidtHieber2020DNN} (see also
\citealp[Ch.~11]{GyorfiKohlerKrzyzakWalk2002}). The corresponding
high-probability approximate-ERM bound has $\delta_n$ in place of
$\E\delta_n$ and yields the $O_P$ rate used in the theorem without an
expectation assumption on $\delta_n$. A fixed pointwise clipping map is
$1$-Lipschitz and therefore cannot increase a class's sup-norm covering
number. It does not increase approximation error only if the target lies
inside the clipping interval. Outcome predictions are clipped to $[-M,M]$;
the clipping thresholds for $e$ and $r$ are chosen within their respective
true overlap bounds.

For Bernoulli cross-entropy, let $p_0(v)=\Pr(Y=1\mid V=v)$ and suppose
both $p_0(v)$ and $p(v)$ lie in $[\epsilon,1-\epsilon]$. Then
\[
 \E\{\ell_{\rm ce}(Y,p(V))-\ell_{\rm ce}(Y,p_0(V))\}
 =\E\!\left[\operatorname{KL}\{\operatorname{Bern}(p_0(V))\|\operatorname{Bern}(p(V))\}\right]
 \ge 2\|p-p_0\|_{L^2(P)}^2,
\]
where the inequality is Pinsker's inequality. Pinsker alone does not
establish an estimation rate. On this clipping interval, cross-entropy
losses are bounded and Lipschitz, and their excess losses satisfy a
Bernstein condition. Localizing their entropy bound gives, for an
approximate empirical minimizer over a class $\mathcal P$ with the entropy
bound of Lemma~\ref{lem:oracle-ineq},
\[
 \|\hat p-p_0\|_{L^2(P)}^2
 =O_P\!\left(\inf_{p\in\mathcal P}\|p-p_0\|_\infty^2
 +\frac{A\log(A_0n)}{n}+\delta_n\right),
\]
with constants depending on $\epsilon$. The sup-norm approximation term
is quadratic because the Bernoulli KL divergence has bounded curvature
on this interval.

The five directly supervised heads fit observed labels: $e$ uses $W$,
$r(w)$ uses $R$ in arm $w$, and $\tilde\mu(w)$ uses $Y$ among labeled
units in arm $w$. The marginal target $\mu_0(w,X)$ is never observed.
Its fit instead uses predictions from the final outcome head on the same
training fold and with a learned baseline encoder. The next lemma controls
this generated-regressor step in terms of the empirical gap
$\delta_{\mu(w),n}$ defined in~\eqref{eq:marginal-block-erm}.
 
\paragraph{The Marginal Head: A Generated-Regressor Lemma. }
The marginal head $\hat\mu(w,\cdot)$ targets
$\mu_0(w,x)=\E[\tilde\mu_0(w,X,S)\mid W{=}w,X{=}x]$, which is itself a
functional of $\tilde\mu_0$: fitting $\hat\mu(w,\cdot)$ is a generated-regressor
problem, with pseudo-outcome $\hat{\tilde\mu}(w,X_i,S_i)$ rather than the true
$\tilde\mu_0(w,X_i,S_i)$. The training loss~\eqref{eq:training-loss}
stop-gradients this pseudo-outcome, but the two heads still use the same
training-fold observations. The uniform bound below handles their
dependence over learned-encoder composite classes.

\begin{lemma}[Learned-encoder same-fold marginal rate]
\label{lem:same-fold-marginal}
Fix $w\in\{0,1\}$ and write $\mathcal Q_{m,w}$ for the bounded,
clipped composite outcome class $(x,s)\mapsto q_w(b(x),g(b(x),s))$;
let $\mathcal F_{m,w}$ be the composite marginal class defined before
Theorem~\ref{thm:composite-direct-rate}. The jointly fitted $\hat q$ and $\hat f$ 
belong to these classes. Let $P_w$ denote expectation given $W=w$.
Suppose their sup-norm covering numbers satisfy
$T_n:=\log n+\log\mathcal N(1/n,\mathcal Q_{m,w},\|\cdot\|_\infty)
+\log\mathcal N(1/n,\mathcal F_{m,w},\|\cdot\|_\infty)$.
For the gap $\delta_{\mu(w),n}$ in~\eqref{eq:marginal-block-erm},
\begin{equation}
 \rho_{\mu(w,\cdot)}^2
 =O_P\!\left(\inf_{f\in\mathcal F_{m,w}}\|f-\mu_0(w,\cdot)\|_\infty^2
 +\frac{T_n}{n}+\delta_{\mu(w),n}
 +\rho_{\tilde\mu(w,\cdot)}^2\right).
 \label{eq:same-fold-marginal-oracle}
\end{equation}
\end{lemma}

\begin{proof}
Write $q_0=\tilde\mu_0(w,\cdot)$ and
$q_0(X,S)=\mu_0(w,X)+\xi$, where $\E[\xi\mid W=w,X]=0$.
Choose a deterministic comparator $f^*\in\mathcal F_{m,w}$ and set
$a=\|f^*-\mu_0(w,\cdot)\|_{L^2(P_w)}$. For
$q\in\mathcal Q_{m,w}$ and $f\in\mathcal F_{m,w}$, set
\[
 D_{q,f}:=(q-f)^2-(q-f^*)^2.
\]
Because all functions are clipped, $D_{q,f}$ is bounded and
$\operatorname{Var}_w(D_{q,f})\le C\,P_w(f-f^*)^2$. A product cover of
$\mathcal Q_{m,w}\times\mathcal F_{m,w}$ at sup-norm resolution $1/n$
has log cardinality at most $T_n$.
For a fixed element of that cover, Bernstein's inequality with a union
bound over the cover gives, simultaneously and with probability tending
to one,
\[
 |(P_w-P_{n,w})D_{q,f}|
 \le C\sqrt{P_w(f-f^*)^2\,T_n/n}+CT_n/n.
\]
Here $n_w\asymp n$ by treatment overlap. The squared-loss difference is
Lipschitz in $(q,f)$ on the clipped range, so passing from the cover to
the full class adds $O(1/n)$. Applying Young's inequality to the square-root
term gives
\begin{equation}
 \left|(P_w-P_{n,w})D_{q,f}\right|
 \le \tfrac18P_w(f-f^*)^2+
 C\frac{T_n}{n}
 \label{eq:joint-localized-bound}
\end{equation}
simultaneously for all $q,f$, with probability tending to one. This step is
uniform over the composite classes, so it applies to $\hat q$ and $\hat f$
even when their encoders were fitted on these observations.

Apply~\eqref{eq:joint-localized-bound} and
\eqref{eq:marginal-block-erm} with $q=\hat q,f=\hat f$. Put
$e=\hat q-q_0$. Since $\E[\xi\mid W=w,X]=0$,
\[
 P_w\{(\hat q-\hat f)^2-(\hat q-f^*)^2\}
 =P_w(\hat f-\mu_0)^2-a^2+2P_w\{e(f^*-\hat f)\}
 \ge \tfrac34P_w(\hat f-\mu_0)^2-Ca^2-C\|e\|_{L^2(P_w)}^2,
\]
where Cauchy--Schwarz and Young's inequality bound the term involving $e$.
Also $P_w(\hat f-f^*)^2\le 2P_w(\hat f-\mu_0)^2+2a^2$.
Combining these bounds with the empirical excess risk and absorbing the
$\tfrac18P_w(\hat f-f^*)^2$ term yields
$P_w(\hat f-\mu_0)^2=O_P(a^2+T_n/n+\delta_{\mu(w),n}
+\|e\|_{L^2(P_w)}^2)$. Minimize over deterministic $f^*$ and use
treatment overlap to transfer the $X$-only marginal-head error from
$P_w$ to $P$:
$P(\hat f-\mu_0)^2\le\Pr(W=w)P_w(\hat f-\mu_0)^2/\epsilon$.
Keep the outcome-head error under $P_w$, as in the definition of
$\rho_{\tilde\mu(w,\cdot)}$. This proves~\eqref{eq:same-fold-marginal-oracle}.
The comparator may use a baseline encoder different from the fitted one because it belongs
to the full class $\mathcal F_{m,w}$, while $\hat q$ remains fixed in
the empirical comparison.
\end{proof}

Outer cross-fitting separates nuisance fitting from score evaluation; it
does not separate the outcome and marginal fits on the same training fold.
The uniform bound handles that dependence, while the assumed final-fit
optimization comparison supplies the otherwise missing marginal comparator.

\subsection{Proof of Proposition \ref{prop:score-robustness}}

\begin{proof}
Fix $w\in\{0,1\}$, and abbreviate
$q_w^\star(X,S)=\tilde\mu^\star(w,X,S)$ and
$q_{0w}(X,S)=\tilde\mu_0(w,X,S)$. Conditional MAR gives
\[
 \E\{R(Y-q_w^\star)\mid W=w,X,S\}
 =r_0(w,X,S)(q_{0w}-q_w^\star).
\]
Indeed, $\E(Y\mid W=w,X,S,R=1)=q_{0w}$ by definition, and
$\E(RY\mid W=w,X,S)=r_0q_{0w}$. Thus the conditional mean of the
arm-$w$ correction in the score, given $X$, is
\begin{align*}
 &\E\left[\left.
 \frac{\ind{W=w}}{e_w^\star(X)}
 \left\{\frac{R(Y-q_w^\star)}{r^\star(w,X,S)}
   +q_w^\star-\mu^\star(w,X)\right\}\right|X\right]\\
 &\quad=\frac{e_{0w}(X)}{e_w^\star(X)}
 \E\left[\left.
 \frac{r_0(w,X,S)}{r^\star(w,X,S)}(q_{0w}-q_w^\star)
   +q_w^\star-\mu^\star(w,X)
 \right|W=w,X\right].
\end{align*}
Write $a_w(X)=e_{0w}(X)/e_w^\star(X)$. Since
$\mu_{0w}(X)=\E(q_{0w}\mid W=w,X)$, the sum of
$\mu^\star(w,X)$ and this correction, minus $\mu_{0w}(X)$, equals
\begin{align*}
 &\mu^\star(w,X)-\mu_{0w}(X)
 +a_w\E\left[\left.
 \frac{r_0}{r^\star}(q_{0w}-q_w^\star)
   +q_w^\star-\mu^\star(w,X)\right|W=w,X\right]\\
 &\quad=(1-a_w)\{\mu^\star(w,X)-\mu_{0w}(X)\}
 +a_w\E\left[\left.
 \left(\frac{r_0}{r^\star}-1\right)(q_{0w}-q_w^\star)
 \right|W=w,X\right].
\end{align*}
Here $r_0/r^\star$ abbreviates the ratio evaluated at $(w,X,S)$.
Under ignorability and conditional MAR,
$\tau_0=\E\{\mu_{01}(X)-\mu_{00}(X)\}$. Taking expectations of the
last display and subtracting the control-arm expression from the treated-arm
expression proves~\eqref{eq:score-bias}. In each arm, either equality in
the first displayed pair makes its first product zero, and either equality
in the second pair makes its second product zero, proving the sufficient
condition.
\end{proof}

\subsection{Proof of Theorem~\ref{thm:composite-direct-rate}}

\begin{proof}
Write $s=s_{\mathrm{comp}}$ and take $m$ sufficiently large. Apply the bounded-weight
approximation construction of \citet[Theorem~5]{SchmidtHieber2020DNN}
to each coordinate of $b_0$, each coordinate of $g_0$, and all seven
representation-level heads. For every component $v$ of input dimension
$D_v$ and smoothness $p_v$, use grid size $N=m$ and accuracy parameter
$t=\lceil(1+s)\log_2m\rceil$. Since $p_v/D_v\ge s$, the construction
gives a sup-norm error bounded by
\[
 C\{m2^{-t}+m^{-p_v/D_v}\}=O(m^{-s})
\]
with $O(\log m)$ depth and $O(m\log m)$ active parameters. The
coordinate changes from $[-1,1]$ to $[0,1]$ only affect constants.
Clip the approximated encoder coordinates to $[-1,1]$, the outcome heads
to $[-M,M]$, and the probability heads to intervals within the true
overlap bounds. Since the corresponding true components lie in these
intervals, clipping cannot increase their approximation errors. The true encoder and head maps are Lipschitz
on their compact domains because their Sobolev orders are at least one.
For example, with $\hat b_m,\hat g_m,\hat q_{w,m}$ denoting the
approximants of $b_0,g_0,\tilde\mu_0(w,\cdot)$, respectively,
\begin{align*}
 &\|\hat q_{w,m}(\hat b_m(X),\hat g_m(\hat b_m(X),S))
   -\tilde\mu_0(w,X,S)\|_\infty\\
 &\quad\le \|\hat q_{w,m}-\bar{\tilde\mu}_0(w,\cdot)\|_\infty
  +C\|\hat g_m-g_0\|_\infty+C\|\hat b_m-b_0\|_\infty
  =O(m^{-s}).
\end{align*}
The same triangle inequality applies to the other six composite heads.
This constructs one shared encoder with all heads simultaneously close to
their true composite nuisances; it does not claim that fitted encoder
coordinates converge to these particular approximants. Because the number
of components is fixed, parallelizing the component constructions and
synchronizing their depths with bounded-weight ReLU identity modules yields
a common layered architecture on $(X,S)$
with $O(\log m)$ depth and width $O(m)$ containing joint approximants
with $O(m\log m)$ nonzero bounded weights and sup-norm error $O(m^{-s})$
for each output. Restrict the joint sieve to at most $O(m\log m)$
nonzero parameters in this architecture; the support may vary across
functions. It has $W=O(m^2\log m)$ available parameter slots, but only
$s_m=O(m\log m)$ of them may be nonzero.

Apply the sparse bound of Lemma~\ref{lem:covering-number} on the bounded
original input domain with $L=O(\log m)$, $W=O(m^2\log m)$,
$s_m=O(m\log m)$, fixed parameter bound $B$, and $\delta=1/n$.
For polynomial $m$, it gives log covering number $O(m(\log n)^3)$
for each composite head. A product cover of the five supervised heads
has the same order because their number is fixed. The two head covers in
Lemma~\ref{lem:same-fold-marginal} also have this order.

To apply the joint comparison, let $\mathcal R_{\mathrm{sup}}$ be the
population version of $\widehat{\mathcal R}_{\mathrm{sup}}$ in
\eqref{eq:joint-supervised-gap}: its outcome term is conditional on
$R=1$, while its propensity terms use the full observed-data law. For
any $\theta\in\Theta_n$, the regression and Bernoulli risk identities give
\begin{align*}
 &\mathcal R_{\mathrm{sup}}(\theta)
      -\mathcal R_{\mathrm{sup}}(\boldsymbol\eta_0)\\
 &=\lambda_Y\E[(\tilde\mu_\theta(W,X,S)
                     -\tilde\mu_0(W,X,S))^2\mid R=1]\\
 &\quad+\lambda_R\E\!\left[
   \operatorname{KL}\{\operatorname{Bern}(r_0(W,X,S))\|
                         \operatorname{Bern}(r_\theta(W,X,S))\}\right]\\
 &\quad+\lambda_W\E\!\left[
   \operatorname{KL}\{\operatorname{Bern}(e_0(X))\|
                         \operatorname{Bern}(e_\theta(X))\}\right].
\end{align*}
Here $\mathcal R_{\mathrm{sup}}(\boldsymbol\eta_0)$ denotes the risk
at the true nuisance vector, whether or not that vector belongs to the
sieve. Fixed positive weights, Pinsker's inequality, and treatment and
labeling positivity imply that this excess risk bounds, up to a constant,
\[
 \|e_\theta-e_0\|_{L^2(P)}^2+
 \sum_{w=0}^1\left\{
 \|r_\theta(w)-r_0(w)\|_{L^2(P_w)}^2+
 \|\tilde\mu_\theta(w)-\tilde\mu_0(w)\|_{L^2(P_w)}^2\right\}.
\]
In particular, for the outcome terms the lower bound on $r_0(w,X,S)$
transfers squared error from the labeled distribution to $P_w$.
No overlap between the two arm-specific surrogate distributions is used.

The common approximant constructed above has population supervised
excess risk $O(m^{-2s})$: the squared-loss term is quadratic, and
Bernoulli KL has bounded curvature on the clipping interval. For the
joint class, the localized squared-loss and bounded cross-entropy
arguments of Lemma~\ref{lem:oracle-ineq} apply to the weighted sum of
loss differences. Its variance is bounded by a constant times the
population supervised excess risk, and the product-cover entropy is
$O(m(\log n)^3)$. The labeled sample count is of order $n$ in
probability by positivity; conditioning on this count handles the
normalization of the labeled squared-loss term. Writing
$\Delta(\theta)=\mathcal R_{\mathrm{sup}}(\theta)
-\mathcal R_{\mathrm{sup}}(\boldsymbol\eta_0)$, Bernstein's inequality,
a union bound on the joint cover, and localization give, uniformly over
$\theta\in\Theta_n$, with probability tending to one,
\[
 \left|\{\widehat{\mathcal R}_{\mathrm{sup}}(\theta)
       -\widehat{\mathcal R}_{\mathrm{sup}}(\boldsymbol\eta_0)\}
       -\Delta(\theta)\right|
 \le \tfrac14\Delta(\theta)+C\frac{m(\log n)^3}{n}.
\]
The empirical risks at $\boldsymbol\eta_0$ are well defined even though
the true vector need not belong to $\Theta_n$. Apply this inequality
both to $\hat\theta$ and to the common approximant, using the gap
definition~\eqref{eq:joint-supervised-gap}, to obtain
\[
 \mathcal R_{\mathrm{sup}}(\hat\theta)
      -\mathcal R_{\mathrm{sup}}(\boldsymbol\eta_0)
 =O_P\!\left(m^{-2s}+\frac{m(\log n)^3}{n}
                  +\delta_{\mathrm{sup},n}\right).
\]
Thus the sum of squared errors of the five supervised nuisances has
this bound, without asserting separate empirical optimality for any
supervised head.

For each $w$, the same joint-sieve construction supplies a deterministic
$f^*\in\mathcal F_{m,w}$ with
$\|f^*-\mu_0(w,\cdot)\|_\infty=O(m^{-s})$; this candidate may use an
encoder different from the fitted one. Lemma~\ref{lem:covering-number}
applied to the composite outcome and marginal classes on $(X,S)$ gives
$T_n=O(m(\log n)^3)$. Lemma~\ref{lem:same-fold-marginal} and the
already established outcome-head rate thus imply
\[
 \rho_{\mu(w,\cdot)}^2
 =O_P\!\left(m^{-2s}+\frac{m(\log n)^3}{n}
 +\delta_{\mu(w),n}+\delta_{\mathrm{sup},n}\right).
\]
With $m\asymp(n/(\log n)^3)^{1/(2s+1)}$ and $s>1/2$, both the
approximation and entropy terms are $o(n^{-1/2})$. If the two
optimization gaps are also $o_P(n^{-1/2})$, all seven nuisance errors
are $o_P(n^{-1/4})$, so the product-rate requirements of
Theorem~\ref{thm:km-oracle} follow. Its inference conclusions require
its additional statistical conditions, including $V^*>0$ for exact
Wald coverage.
\end{proof}

\end{document}